\documentclass[11pt]{article}

\usepackage[a4paper,margin=1in]{geometry}
\usepackage{microtype}
\usepackage{graphicx}
\usepackage{subcaption}
\usepackage{booktabs}
\usepackage[colorlinks=true,linkcolor=blue,citecolor=blue,urlcolor=blue]{hyperref}

\usepackage{amsmath,amssymb,mathtools}
\usepackage{amsthm}
\usepackage[capitalize,noabbrev]{cleveref}

\usepackage{algorithm}
\usepackage{algorithmic}

\usepackage{multirow}
\usepackage{siunitx}
\usepackage{authblk}

\usepackage[round,authoryear]{natbib}

\theoremstyle{plain}
\newtheorem{theorem}{Theorem}
\newtheorem{lemma}{Lemma}
\newtheorem{corollary}{Corollary}

\theoremstyle{definition}

\title{Sample Transform Cost-Based Training-Free Hallucination Detector for Large Language Models}

\author[1]{Zeyang Ding}
\author[1]{Xinglin Hu}
\author[1]{Jicong Fan\thanks{Corresponding author: \texttt{fanjicong@cuhk.edu.cn}}}

\affil[1]{The Chinese University of Hong Kong, Shenzhen}
\affil[ ]{\texttt{zding@cuhk.edu.cn}}
\affil[ ]{\texttt{xinglinhu@link.cuhk.edu.cn}}
\affil[ ]{\texttt{fanjicong@cuhk.edu.cn}}

\date{}

\begin{document}

\maketitle

\begin{abstract}
Hallucinations in large language models (LLMs) remain a central obstacle to trustworthy deployment, motivating detectors that are accurate, lightweight, and broadly applicable. Since an LLM with a prompt defines a conditional distribution, we argue that the complexity of the distribution is an indicator of hallucination. However, the density of the distribution is unknown and the samples (i.e., responses generated for the prompt) are discrete distributions, which leads to a significant challenge in quantifying the complexity of the distribution. We propose to compute the optimal-transport distances between the sets of token embeddings of pairwise samples, which yields a Wasserstein distance matrix measuring the costs of transforming between the samples. This Wasserstein distance matrix provides a means to quantify the complexity of the distribution defined by the LLM with the prompt. Based on the Wasserstein distance matrix, we derive two complementary signals: AvgWD, measuring the average cost, and EigenWD, measuring the cost complexity. This leads to a training-free detector for hallucinations in LLMs. We further extend the framework to black-box LLMs via teacher forcing with an accessible teacher model. Experiments show that AvgWD and EigenWD are competitive with strong uncertainty baselines and provide complementary behavior across models and datasets, highlighting distribution complexity as an effective signal for LLM truthfulness.

\end{abstract}

\section{Introduction}
Large language models (LLMs) have rapidly transformed modern artificial intelligence, enabling strong performance in open-ended dialogue, reasoning, and code generation at unprecedented scale \citep{vaswani2017attention,brown2020language,achiam2023gpt,touvron2023llama}. Despite this progress, their reliability remains undermined by \emph{hallucinations}---outputs that are fluent and plausible yet factually unsupported or incorrect \citep{maynez2020faithfulness,ji2023survey,huang2025survey}. This failure mode is particularly concerning in high-stakes settings such as healthcare, education, and scientific workflows, where incorrect but convincing answers can be difficult to detect \citep{lin2022truthfulqa}. At its core, this problem reflects the mismatch between next-token prediction and factual reliability: LLMs are trained to produce continuations that are \emph{likely} under the data distribution, not necessarily statements that are verifiable or grounded in external evidence \citep{welleck2019neural,lin2022truthfulqa}.

A practical hallucination detector should be accurate, lightweight, and broadly applicable across model access regimes. Existing approaches, however, involve clear trade-offs. \emph{External verification} and retrieval-augmented pipelines \citep{lewis2020retrieval} improve factual grounding, but require reliable external knowledge and additional system integration, and can still fail when retrieval is noisy or evidence is misused \citep{gao2023rarr}. \emph{Black-box self-consistency} methods avoid external resources by comparing multiple sampled outputs, but they operate mainly in text space and are sensitive to paraphrasing, stylistic variation, and heuristic scoring choices \citep{manakul2023selfcheckgpt}. \emph{White-box uncertainty} signals derived from logits are efficient, yet they often reflect lexical uncertainty rather than semantic or factual reliability, and may be poorly calibrated under decoding and distribution shift \citep{malinin2020uncertainty,kuhn2023semantic}.

Recent evidence suggests that hallucinations correlate with \emph{internal} inconsistencies that are more salient in representation space than in surface form. Representation-based detectors therefore exploit hidden states to capture semantic variation without relying on external modules \citep{chen2024inside,wang2025revisiting}. A central question is how to represent an entire response. Prior work often compresses a response into a single vector, such as a last-token state or pooled embedding, or relies on spectral statistics of heavily summarized features \citep{chen2024inside}. But in long-form generation, uncertainty is often distributed across many tokens---for example, entity mentions, numerical steps, and local justifications---making such compression potentially brittle.

In this paper, we instead treat multiple sampled responses to a prompt as observations from a prompt-conditioned response distribution, and argue that the \emph{complexity} of this distribution is predictive of hallucination. Because this distribution is unknown and only a finite set of discrete response samples is observed, we compare responses directly at the token level in representation space. For each sampled response, we construct an empirical measure over generated-token embeddings from an intermediate layer and compute pairwise optimal-transport (OT) distances between sampled responses. This produces a Wasserstein distance matrix whose entries quantify the cost of transforming one sampled response into another \citep{cuturi2013sinkhorn}. From this structure, we derive two complementary training-free signals: \textbf{AvgWD}, which measures average transform cost across sampled pairs, and \textbf{EigenWD}, which captures the spectral complexity of the induced cost structure. Together, they define a distribution-consistency detector for hallucination. We further extend the approach to black-box LLMs by using a teacher model to approximate hidden representations via teacher forcing on sampled outputs, preserving the same multi-sample consistency signal without access to the target model's internal states.

Our Contributions are summarized as follows.
\begin{itemize}
    \item We propose a training-free distribution-consistency framework for hallucination detection using generated-token intermediate representations.
    \item We introduce two complementary signals from the resulting Wasserstein structure: AvgWD (cost magnitude) and EigenWD (cost-structure complexity).
    \item We generalize the detector to black-box models via teacher forcing with a teacher LLM, retaining multi-sample consistency without requiring access to the target model’s hidden states.
    \item We evaluate on $5$ open-source LLMs and $4$ datasets (20 model--dataset settings) and compare against $5$ training-free baselines; our method achieves the best overall performance on average and is top-performing in many settings.
\end{itemize}

\section{Related Work}
\label{sec:related}

\subsection{Hallucination Evaluation and Detection}
Hallucination has been studied extensively in summarization and open-ended generation, where factuality and faithfulness errors can be subtle and hard to assess automatically \citep{maynez2020faithfulness,ji2023survey}. Beyond task-specific benchmarks such as TruthfulQA \citep{lin2022truthfulqa}, recent efforts curate large-scale hallucination evaluation resources for LLMs, e.g., HaluEval, which provides human-annotated hallucinated samples covering diverse topics and error types \citep{li2023halueval}. In parallel, fine-grained factuality evaluation for long-form generation has been advanced by atomic fact decomposition and evidence-backed scoring, such as FActScore \citep{min2023factscore}. While these benchmarks and evaluators are primarily designed for assessment, they influence detector design by clarifying what constitutes hallucination and how detection performance is measured.

Hallucination detection methods broadly fall into two categories. \emph{External verification} frameworks interface LLMs with retrieval and evidence attribution, and then revise or validate generations based on retrieved sources \citep{gao2023rarr}. These approaches can be effective for grounded tasks but incur system complexity and depend on the quality of retrievers and knowledge sources. \emph{Self-consistency} approaches avoid external modules by comparing multiple generations for the same prompt, based on the intuition that correct knowledge yields consistent outputs while hallucinations lead to divergence \citep{manakul2023selfcheckgpt}. Our method follows the multi-sample consistency principle but moves the comparison into representation space, aiming to reduce sensitivity to surface-form variation while also enabling structural (spectral) analysis of cross-sample discrepancies.

\subsection{Uncertainty Quantification and Self-Evaluation in LLMs}
A large body of work uses probabilistic signals (e.g., token entropy, sequence probability, or derived confidence measures) as uncertainty proxies \citep{malinin2020uncertainty,kendall2017uncertainties,gal2016dropout,lakshminarayanan2017simple}. However, natural language exhibits semantic invariances: multiple strings can express the same meaning, making purely lexical uncertainty insufficient. Semantic uncertainty methods address this by grouping generations by meaning and defining uncertainty over semantic equivalence classes (e.g., semantic entropy) \citep{kuhn2023semantic}. Complementary to these, self-evaluation studies show that LLMs can sometimes predict correctness when prompted appropriately, suggesting that models may contain internal signals about their own knowledge and uncertainty \citep{kadavath2022language}. Our approach differs in that it does not require eliciting explicit self-evaluation outputs; instead, it measures (i) the magnitude and (ii) the structural complexity of inconsistency in intermediate representations across multiple samples.

\subsection{Representation-Based Hallucination Detection}
Representation-based detectors use hidden states to capture semantic variation and uncertainty from a model’s internal computations. \citet{chen2024inside} proposed INSIDE, an eigenscore based on spectral statistics across multiple generations. The Effective Rank-based Uncertainty proposed by \citep{wang2025revisiting} further motivates spectrum-based measures over responses and layers. In contrast, we treat an LLM with a prompt as defining a conditional distribution and quantify its \emph{distribution complexity} by computing pairwise Wasserstein distances between generated-token embedding sets. 

\subsection{Black-box Settings and Teacher Forcing}
In many real deployments, internal states of proprietary LLMs are unavailable, motivating black-box detectors. Sampling-based detectors are widely applicable in black-box settings, as they rely only on model outputs. However, they remain limited to text-level signals \citep{manakul2023selfcheckgpt}. A complementary direction uses an accessible teacher model to extract features via teacher forcing on the target model’s outputs, enabling representation-aware signals even in black-box regimes \citep{sriramanan2024llm}. Our black-box extension follows this teacher-forcing paradigm while preserving the multi-sample consistency signal central to self-consistency methods, resulting in a unified framework for both access regimes.

\section{Methodology}
\label{sec:method}

We develop a training-free hallucination detector based on sample-to-sample transform costs in representation space \citep{peyre2019computational,cuturi2013sinkhorn}.
For a given prompt $x$, we draw $K$ stochastic responses and compute pairwise optimal-transport distances between their generated-token embeddings, forming a Wasserstein distance matrix that quantifies the cost of transforming one sampled response into another.
From this matrix, we derive two complementary signals: \textbf{AvgWD}, which captures the average transform cost, and \textbf{EigenWD}, which captures the complexity of the induced cost structure via spectral statistics.
Together, AvgWD and EigenWD define a training-free detector for hallucinations.

\paragraph{Notation.}
Let $\mathrm{LLM}_\theta$ denote a target LLM with parameters $\theta$. Given a prompt $x$, we draw $K$ stochastic generations
\begin{equation}
y_i \sim p_\theta(\cdot\mid x),\quad i=1,\dots,K,
\label{eq:sampling}
\end{equation}
where each response $y_i=(y_{i,1},\dots,y_{i,n_i})$ is a token sequence of length $n_i$. Fix an intermediate layer $\ell$ and let
\begin{equation}
\mathbf{z}_{i,t}^{(\ell)} \in \mathbb{R}^d
\label{eq:zit}
\end{equation}
denote the layer-$\ell$ hidden state for token position $t$ in $y_i$, where $t=1,\ldots,n_i$.

\paragraph{Generating Multiple Responses via Stochastic Decoding.}
In practice, these responses are obtained via \emph{stochastic decoding}, by sampling tokens from a temperature-scaled and optionally truncated next-token distribution. At decoding step $t$,
\begin{equation}
y_t \sim \mathrm{Cat}\!\left(\tilde{p}_\theta(\cdot \mid x, y_{<t})\right),\quad
\tilde{p}_\theta(v \mid \cdot)\propto \exp\!\left(\ell_v/\tau\right),
\end{equation}
where $\ell_v$ is the logit for token $v$ and $\tau>0$ is the temperature. Here, $p_\theta(\cdot\mid x,y_{<t})$ denotes the model's original next-token distribution, while $\tilde p_\theta(\cdot\mid x,y_{<t})$ denotes the sampling distribution after temperature scaling and optional truncation followed by renormalization. Truncation may be implemented with top-$k$ or nucleus sampling.

\paragraph{Overview.}
\Cref{fig:avgwd_overview} illustrates the overall pipeline: sample multiple responses for the same prompt, extract generated-token hidden states, compute pairwise Wasserstein discrepancies, and summarize the resulting structure either by its average magnitude (AvgWD) or by its spectral complexity (EigenWD). We summarize the white-box procedure in \Cref{alg:whitebox}.

\begin{figure*}[t]
  \centering
  \includegraphics[width=0.98\textwidth]{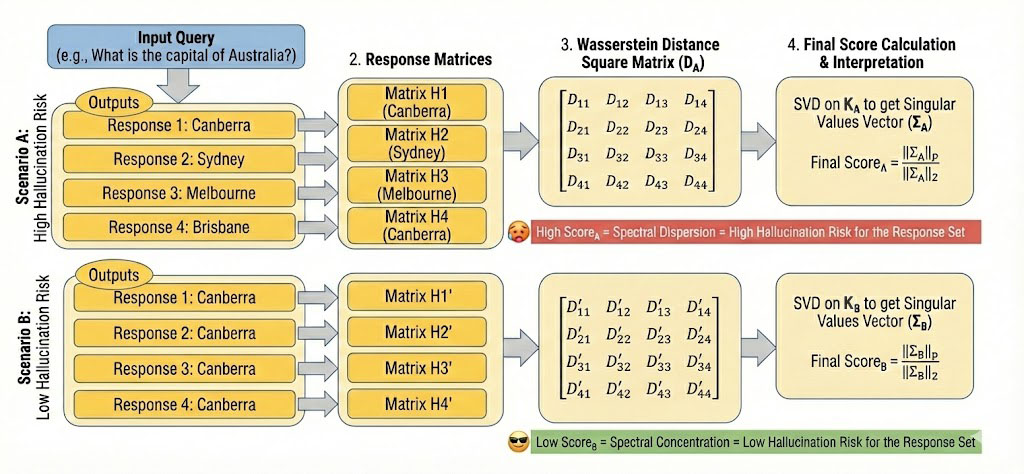}
  \caption{\textbf{Overview of distribution-consistency detection.}
  For a given prompt, we sample multiple responses, extract generated-token hidden states, compute pairwise Wasserstein distances between responses, and summarize the resulting structure using AvgWD (average pairwise transform cost) and EigenWD (spectral complexity of the induced cost matrix).}
  \label{fig:avgwd_overview}
\end{figure*}

\begin{algorithm}[t]
\caption{White-box AvgWD/EigenWD}
\label{alg:whitebox}
\begin{algorithmic}[1]
\STATE \textbf{Input:} prompt $x$; white-box LLM $\mathrm{LLM}_\theta$; layer $\ell$; number of samples $K$.
\STATE \textbf{Output:} $\mathrm{AvgWD}(x)$ and $\mathrm{EigenWD}(x)$.
\STATE Sample $K$ stochastic responses $\{y_i\}_{i=1}^K$ from $\mathrm{LLM}_\theta$ given $x$.
\FOR{$i=1$ \textbf{to} $K$}
  \STATE Extract generated-token states at layer $\ell$: $Z_i^{(\ell)}=\{\mathbf{z}_{i,t}^{(\ell)}\}_{t=1}^{m_i}$.
  \STATE Construct the empirical measure
  $\mu_i^{(\ell)} \leftarrow \frac{1}{m_i}\sum_{t=1}^{m_i}\delta_{\mathbf{z}_{i,t}^{(\ell)}}$.
\ENDFOR
\FOR{$i=1$ \textbf{to} $K$}
  \FOR{$j=i+1$ \textbf{to} $K$}
    \STATE Compute exact OT with squared $\ell_2$ ground cost and set
    $D_{ij}\leftarrow W_2(\mu_i^{(\ell)},\mu_j^{(\ell)})$.
    \STATE $D_{ji}\leftarrow D_{ij}$.
  \ENDFOR
  \STATE $D_{ii}\leftarrow 0$.
\ENDFOR
\STATE $\mathrm{AvgWD}(x)\leftarrow \frac{2}{K(K-1)}\sum_{i<j} D_{ij}$.
\STATE $\mathrm{EigenWD}(x)\leftarrow \mathrm{SpectralComplexity}(D)$ (see \Cref{subsec:eigenwd}).
\STATE \textbf{return} $\mathrm{AvgWD}(x), \mathrm{EigenWD}(x)$.
\end{algorithmic}
\end{algorithm}

\subsection{Generated-token embeddings as empirical distributions}
We represent each response by an empirical measure over its generated-token embeddings:
\begin{equation}
\mu_i^{(\ell)} \;=\; \frac{1}{m_i}\sum_{t=1}^{m_i} \delta_{\mathbf{z}_{i,t}^{(\ell)}},
\label{eq:empirical_measure}
\end{equation}
where $\delta_u$ is the Dirac measure at $u\in\mathbb{R}^d$, and $m_i$ is the number of generated tokens retained in the empirical support after preprocessing. Equivalently, $\mu_i^{(\ell)}$ is the uniform empirical distribution over the token embeddings $\{\mathbf{z}^{(\ell)}_{i,t}\}_{t=1}^{m_i}$.

\paragraph{Implementation details.}
We (i) use only the generated continuation tokens (excluding the prompt segment), (ii) exclude the EOS token from the support of $\mu_i^{(\ell)}$, and (iii) assign uniform token weights.

\paragraph{Mid-layer choice and projection.}
To reduce layer sensitivity and match the implementation, we use the middle transformer layer $\ell=\lfloor L/2\rfloor$, where $L$ is the number of layers. We also apply a fixed shared random projection to 128 dimensions to reduce OT cost while preserving the core signal.

\subsection{AvgWD: Average Sample Transform Cost as Distribution Complexity}
\label{subsec:avgwd}

Given a prompt $x$, an LLM induces a conditional distribution $p_\theta(\cdot\mid x)$ over responses.
The \emph{complexity} of this distribution is reflected by how costly it is, on average, to transport one sampled response into another in representation space.
The main difficulty is that the density of $p_\theta(\cdot\mid x)$ is unknown and we only observe a finite set of samples $\{y_i\}_{i=1}^K$, where each $y_i$ is a variable-length sequence rather than a fixed-dimensional point.
Therefore, classical complexity surrogates based on estimating moments (e.g., covariance) are not directly applicable without additional fixed-dimensional summarization.
Instead of pooling a response into a single vector (which discards token-level structure), we represent each response by an empirical distribution over its generated-token embeddings (Sec.~3.1) and measure sample transform costs between these distributions.

\paragraph{Wasserstein distance (continuous definition).}
Let $\mu$ and $\nu$ be probability measures on $\mathbb{R}^d$ with finite second moments.
The 2-Wasserstein distance is
\begin{equation}
\mathcal{W}_2(\mu,\nu)
\;:=\;
\Big(
\inf_{\gamma\in\Pi(\mu,\nu)}
\int_{\mathbb{R}^d\times\mathbb{R}^d}
\|u-v\|_2^2\, d\gamma(u,v)
\Big)^{\frac12},
\label{eq:w2_cont}
\end{equation}
where $\Pi(\mu,\nu)$ denotes the set of couplings whose marginals are $\mu$ and $\nu$ \citep{peyre2019computational}.

\paragraph{Discrete (empirical) version used in our detector.}
For each sampled response $y_i$, we form an empirical measure over its generated continuation token embeddings at layer $\ell$:
\begin{equation}
\mu_i^{(\ell)}=\frac{1}{m_i}\sum_{t=1}^{m_i}\delta_{\mathbf{z}^{(\ell)}_{i,t}},
\label{eq:emp_measure}
\end{equation}
where $m_i$ is the number of generated continuation tokens retained in the support (excluding the prompt segment and EOS), and we assign uniform token weights.
For two empirical measures $\mu_i^{(\ell)}$ and $\mu_j^{(\ell)}$, the discrete OT problem becomes a finite-dimensional linear program.
Let $\mathbf{a}\in\mathbb{R}^{m_i}$ and $\mathbf{b}\in\mathbb{R}^{m_j}$ be uniform weights, i.e., $a_t=\tfrac{1}{m_i}$ and $b_s=\tfrac{1}{m_j}$.
With the squared Euclidean ground cost $c(\mathbf{u},\mathbf{v})=\|\mathbf{u}-\mathbf{v}\|_2^2$, the squared EMD objective is
\begin{equation}
\begin{aligned}
\mathrm{EMD2}\left(\mu_i^{(\ell)},\mu_j^{(\ell)};c\right)
\;:=\;
\min_{\mathbf{P}\in\mathbb{R}_+^{m_i\times m_j}}
&\;\sum_{t=1}^{m_i}\sum_{s=1}^{m_j} P_{ts}\,
c\!\left(\mathbf{z}^{(\ell)}_{i,t},\mathbf{z}^{(\ell)}_{j,s}\right) \\
\text{s.t.}\quad
& \mathbf{P}\mathbf{1}=\mathbf{a},\qquad \mathbf{P}^\top\mathbf{1}=\mathbf{b}.
\end{aligned}
\label{eq:emd2_disc}
\end{equation}
In our implementation, we compute $\mathrm{EMD2}$ exactly using POT \texttt{emd2} and take the square root to obtain the 2-Wasserstein distance:
\begin{equation}
D_{ij}
\;:=\;
\sqrt{\mathrm{EMD2}(\mu_i^{(\ell)},\mu_j^{(\ell)};c)}.
\label{eq:Dij_def}
\end{equation}
Collecting all pairwise distances yields a symmetric matrix $\mathbf{D}\in\mathbb{R}^{K\times K}$ with $D_{ii}=0$ and $D_{ij}=D_{ji}$.

\paragraph{AvgWD as a U-statistic.}
We define \textbf{AvgWD} as the average sample transform cost across all unordered pairs:
\begin{equation}
\mathrm{AvgWD}(x)
:= \frac{2}{K(K-1)}\sum_{1\le i<j\le K} D_{ij}.
\label{eq:avgwd}
\end{equation}
Under independent draws from $p_\theta(\cdot\mid x)$, this is a standard second-order U-statistic over the $K$ samples, providing an unbiased estimator of the expected pairwise transform cost and, under standard regularity conditions, an asymptotically normal estimator as $K\to\infty$.
Intuitively, a larger AvgWD means that sampled responses are, on average, more expensive to transport into one another in representation space, indicating higher distribution complexity and hence a higher risk of hallucination.

\subsection{EigenWD: Spectral Complexity of Transform-Cost Structure}
\label{subsec:eigenwd}

AvgWD summarizes the \emph{magnitude} of sample transform costs (Sec.~\ref{subsec:avgwd}).
We further characterize the \emph{structural complexity} encoded by the full pairwise transform-cost matrix $\mathbf{D}\in\mathbb{R}^{k\times k}$.

\paragraph{From costs to a similarity structure.}
We convert $\mathbf{D}$ into a similarity (kernel) matrix $\mathbf{K}$ using a Gaussian kernel,
\begin{equation}
K_{ij}=\exp\!\left(-\frac{D_{ij}^2}{2\,(b^2+\epsilon)}\right),
\label{eq:kernel}
\end{equation}
where $\epsilon$ is a small constant for numerical stability. Following the implementation, we set the bandwidth as
\begin{equation}
b=\mathrm{median}\{D_{ij}: D_{ij}>0\},
\label{eq:bandwidth}
\end{equation}
and use $\epsilon=10^{-6}$ (defaulting to $b=1$ if no positive entry exists). To improve numerical robustness, we add a small diagonal shift,
\begin{equation}
\mathbf{K} \leftarrow \mathbf{K} + \alpha \mathbf{I},
\label{eq:kernel_stab}
\end{equation}
with a small $\alpha>0$ as used in the implementation.

\paragraph{Why kernelize before spectral analysis.}
Although $\mathbf{D}$ contains the raw transform costs, its spectrum is dominated by global scale and can be sensitive to outliers. Consequently, the spectrum of $\mathbf{D}$ is less directly interpretable as structural complexity.
In contrast, the kernelization in Eq.~\eqref{eq:kernel} maps costs into a bounded similarity structure with controlled scale via $b$, making the resulting spectrum more stable and more directly tied to how responses organize into modes (e.g., coherent clusters versus fragmented multi-modal inconsistencies).

\paragraph{EigenWD definition.}
Let the eigenvalue decomposition of $\mathbf{K}$ be
\begin{equation}
\mathbf{K}=\mathbf{V}\mathbf{\Lambda}\mathbf{V}^\top,
\end{equation}
where $\mathbf{\Lambda}=\mathrm{diag}(\lambda_1,\lambda_2,\ldots,\lambda_k)$ and $\lambda_i$ denotes the $i$-th eigenvalue of $\mathbf{K}$.
We then define
\begin{equation}
\mathrm{EigenWD}(x)=\frac{\|\boldsymbol{\lambda}\|_{p}}{\|\boldsymbol{\lambda}\|_{2}}
=\frac{\left(\sum_{i=1}^k \lambda_i^p\right)^{1/p}}{\left(\sum_{i=1}^k \lambda_i^2\right)^{1/2}},
\label{eq:eigenwd}
\end{equation}
where $0<p<2$, and $\|\cdot\|_p$ denotes the $\ell_p$ quasi-norm when $0<p<1$.
This quantity is scale-invariant because both numerator and denominator scale linearly with $\boldsymbol{\lambda}$.
For $p<2$ and nonzero $\boldsymbol{\lambda}$, we have $\mathrm{EigenWD}(x)\ge 1$, with equality only in the rank-one case.
Smaller values of $p$ place greater emphasis on the spread of small but non-negligible eigenvalues.
Accordingly, $\mathrm{EigenWD}(x)$ increases when the spectrum is less concentrated, indicating a more complex transform-cost structure across sampled responses.
Our implementation uses $p=0.1$.

\section{Robustness of AvgWD}
\label{sec:robustness}

We provide a robustness guarantee for AvgWD under token-level perturbations of the extracted hidden states. Intuitively, if the token embeddings of each sampled response are perturbed slightly in Frobenius norm, then the resulting pairwise Wasserstein distances change slightly, and therefore AvgWD changes slightly.

\paragraph{Setup.}
Fix a layer $\ell$ and suppress $(\ell)$ for simplicity. For each sampled response $i\in\{1,\dots,K\}$, collect the generated-token embeddings into a matrix
\begin{equation}
\mathbf{Z}_i \;=\; 
\begin{bmatrix}
\mathbf{z}_{i,1}^\top \\
\vdots \\
\mathbf{z}_{i,m_i}^\top
\end{bmatrix}
\in\mathbb{R}^{m_i\times d},
\end{equation}
and define the associated empirical measure
\begin{equation}
\mu_i=\frac{1}{m_i}\sum_{t=1}^{m_i}\delta_{\mathbf{z}_{i,t}}.
\end{equation}
Consider perturbed embeddings $\{\mathbf{Z}'_i\}_{i=1}^K$ such that, for each $i$, the perturbed sample retains the same support size $m_i$. Let $\mu_i'$ denote the corresponding perturbed empirical measure, and define
\begin{equation}
D_{ij}=W_2(\mu_i,\mu_j),\qquad D'_{ij}=W_2(\mu'_i,\mu'_j).
\end{equation}

\begin{lemma}[Two-sided Wasserstein stability]
\label{lem:two_sided_w2}
For any probability measures $\mu,\nu,\mu',\nu'$ on $\mathbb{R}^d$,
\begin{equation}
\bigl|W_2(\mu,\nu)-W_2(\mu',\nu')\bigr|
\;\le\;
W_2(\mu,\mu') + W_2(\nu,\nu').
\label{eq:two_sided_w2}
\end{equation}
\end{lemma}

\begin{lemma}[Token-level perturbation bound]
\label{lem:token_w2}
Let $\mathbf{Z},\mathbf{Z}'\in\mathbb{R}^{m\times d}$ and let
\begin{equation}
\mu(\mathbf{Z})=\frac{1}{m}\sum_{t=1}^m \delta_{\mathbf{z}_t},
\qquad
\mu(\mathbf{Z}')=\frac{1}{m}\sum_{t=1}^m \delta_{\mathbf{z}'_t}
\end{equation}
be the corresponding uniform empirical measures. Then
\begin{equation}
W_2\bigl(\mu(\mathbf{Z}),\mu(\mathbf{Z}')\bigr)
\;\le\;
\frac{\|\mathbf{Z}-\mathbf{Z}'\|_F}{\sqrt{m}}.
\label{eq:token_w2}
\end{equation}
\end{lemma}

The proofs of \Cref{lem:two_sided_w2,lem:token_w2} are deferred to Appendix~\ref{app:stability_proofs}.

\begin{theorem}[AvgWD is Lipschitz under token-level perturbations]
\label{thm:avgwd_lipschitz}
Define the per-sample perturbation magnitudes
\begin{equation}
\varepsilon_i \;=\; \frac{\|\mathbf{Z}_i - \mathbf{Z}'_i\|_F}{\sqrt{m_i}},
\qquad i=1,\dots,K.
\label{eq:eps_def}
\end{equation}
Then
\begin{equation}
\bigl|\mathrm{AvgWD}(\mathbf{Z}_{1:K})-\mathrm{AvgWD}(\mathbf{Z}'_{1:K})\bigr|
\;\le\;
\frac{2}{K}\sum_{i=1}^K \varepsilon_i.
\label{eq:avgwd_lip}
\end{equation}
\end{theorem}

\noindent\textbf{Proof.} Deferred to Appendix~\ref{app:stability_proofs}.

\begin{corollary}
\label{cor:avgwd_uniform_eps}
If $\varepsilon_i\le \varepsilon$ for all $i$, then
\begin{equation}
\bigl|\mathrm{AvgWD}(\mathbf{Z}_{1:K})-\mathrm{AvgWD}(\mathbf{Z}'_{1:K})\bigr|
\le 2\varepsilon.
\end{equation}
\end{corollary}

\paragraph{Remark on EigenWD stability.}
EigenWD is obtained by composing kernelization of $\mathbf{D}$, spectral decomposition of the resulting kernel matrix, and a ratio of spectral quasi-norms. With the diagonal shift $\mathbf{K}\leftarrow \mathbf{K}+\alpha\mathbf{I}$ used in our implementation, the kernel matrix is better conditioned, and EigenWD is locally Lipschitz with respect to $\|\mathbf{D}-\mathbf{D}'\|_F$ under mild boundedness assumptions. A concrete statement and proof are given in Appendix~\ref{app:stability_proofs}.

\section{Experiments}
\label{sec:experiments}
\subsection{Experimental Settings}\label{sec:setup}

\paragraph{Datasets.}
We evaluate in the white-box setting on two extractive QA benchmarks, \textbf{CoQA} \citep{reddy2019coqa}
and \textbf{SQuAD} \citep{rajpurkar2016squad}, and further consider two additional domains:
mathematical reasoning on \textbf{MATH-500} (a subset of \textbf{MATH}) \citep{hendrycks2021measuring}
and abstractive summarization on \textbf{CNN/DailyMail} \citep{hermann2015teaching,nallapati2016abstractive}.
Each example consists of a prompt (question/problem/document) and a corresponding reference (answer/summary).
For each prompt, we sample $K$ stochastic generations from the target model and perform hallucination detection at the \emph{prompt level}
(i.e., whether the model's response is correct/grounded w.r.t.\ the reference answer or reference summary).

\paragraph{Label construction.}
We derive binary correctness labels automatically.
For extractive QA benchmarks \textbf{CoQA} and \textbf{SQuAD}, we compute ROUGE-L \citep{lin2004rouge}
between each generated answer and the reference answer,
and label the response as \emph{correct} if $\text{ROUGE-L} \ge 0.5$, and \emph{hallucinated} otherwise.
For \textbf{Math-500} and \textbf{CNN/DailyMail}, where exact-match style string comparison is not reliable (free-form numeric answers and open-ended summaries),
we use an LLM-based judge: \textbf{GPT-4o} is prompted with the input and the model response (and reference, when available) to output a binary correctness decision.
This follows the common ``LLM-as-a-judge'' evaluation paradigm \citep{liu2023g,min2023factscore}, using a GPT-4-class model \citep{achiam2023gpt}.
We then use these labels for prompt-level hallucination detection evaluation.

\paragraph{Models.}
We run experiments on open-source LLMs under the white-box access regime. Concretely, we report results for
\textbf{Llama-3.2-3B}, \textbf{Llama-2-7B} \citep{touvron2023llama}, \textbf{Llama-3.1-8B}, \textbf{Qwen3-8B}, and \textbf{Qwen3-32B}.
All models are evaluated with identical decoding and scoring pipelines (same temperature and truncation settings, the same number of sampled responses $K$, and the same embedding layer)
to ensure fair comparison across methods.

\paragraph{Evaluation metric.}
We treat each detector as producing a real-valued hallucination score per prompt, and report
\textbf{AUROC} (area under the ROC curve; higher is better), which is threshold-free and standard for binary detection.

\paragraph{Baselines.}
We compare AvgWD/EigenWD against strong training-free uncertainty baselines:
(i) \textbf{Discrete Semantic Entropy (DSE)} \citep{farquhar2024detecting,kuhn2023semantic};
(ii) \textbf{Eigenscore (ES)} \citep{chen2024inside};
(iii) \textbf{Length-Normalized Entropy (LNE)} \citep{malinin2020uncertainty,kadavath2022language};
(iv) \textbf{Lexical Similarity (LS)} \citep{manakul2023selfcheckgpt};
and (v) \textbf{Effective Rank (ER)} \citep{roy2007effective,wang2025revisiting}.
All baseline methods use the same set of sampled generations per prompt.

\paragraph{Implementation details.}
We sample $K=10$ responses per prompt with temperature $\tau=0.5$.
We extract hidden-state embeddings from the middle layer $\ell=\lfloor L/2\rfloor$
(where $L$ is the number of transformer layers), and form token-level measures using
generation-continuation tokens only (excluding the prompt segment and EOS), with uniform token weights.
We compute pairwise OT discrepancies following standard optimal transport formulations \citep{peyre2019computational}
via exact EMD with squared $\ell_2$ ground cost and take the square root
to obtain $W_2$ distances. All experiments are run on an NVIDIA RTX PRO 6000 GPU with 96GB memory.

% -------------------------
% Fig. 3: CoQA heatmaps
% -------------------------
\begin{figure}[t]
  \centering
  \begin{subfigure}[t]{0.48\linewidth}
    \centering
    \includegraphics[width=\linewidth]{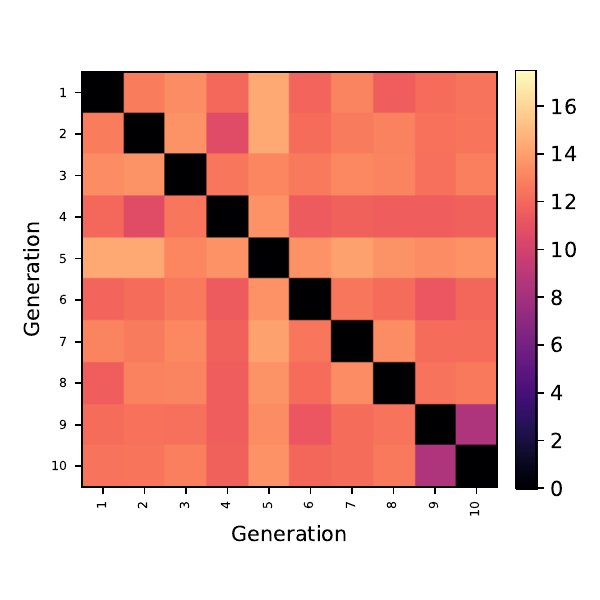}
    \caption{CoQA (hallucination)}
  \end{subfigure}\hfill
  \begin{subfigure}[t]{0.48\linewidth}
    \centering
    \includegraphics[width=\linewidth]{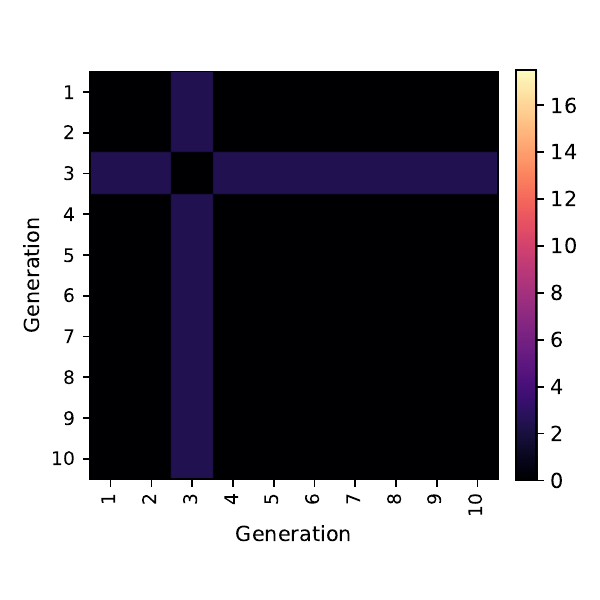}
    \caption{CoQA (non-hallucination)}
  \end{subfigure}
  \caption{\textbf{Heatmaps of sample-to-sample OT costs on CoQA.}
  For a prompt, $D_{ij}=W_2(\mu_i,\mu_j)$ is the Wasserstein distance between the $i$-th and $j$-th sampled responses (diagonal is zero);
  darker cells indicate larger transform costs. Hallucinated cases often exhibit larger average costs and/or more fragmented block structure,
  motivating AvgWD (magnitude) and EigenWD (structure).}
  \label{fig:ot_coqa_singlecol}
  \vspace{-10pt}
\end{figure}

% -------------------------
% Fig. 3: SQuAD heatmaps
% -------------------------
\begin{figure}[t]
  \centering
  \begin{subfigure}[t]{0.48\linewidth}
    \centering
    \includegraphics[width=\linewidth]{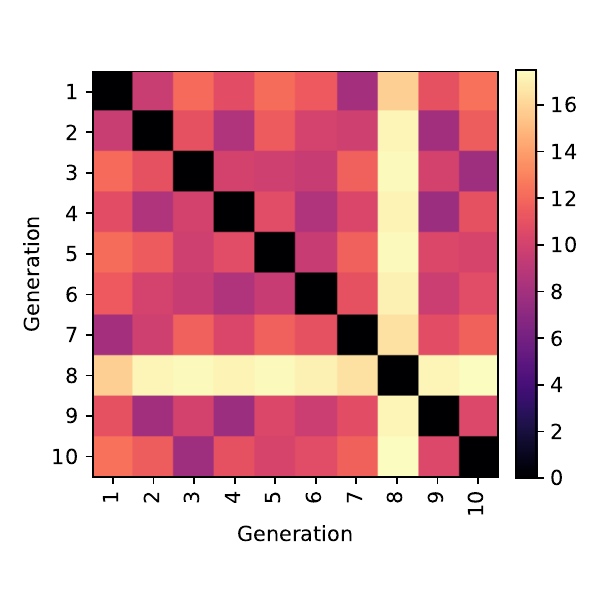}
    \caption{SQuAD (hallucination)}
  \end{subfigure}\hfill
  \begin{subfigure}[t]{0.48\linewidth}
    \centering
    \includegraphics[width=\linewidth]{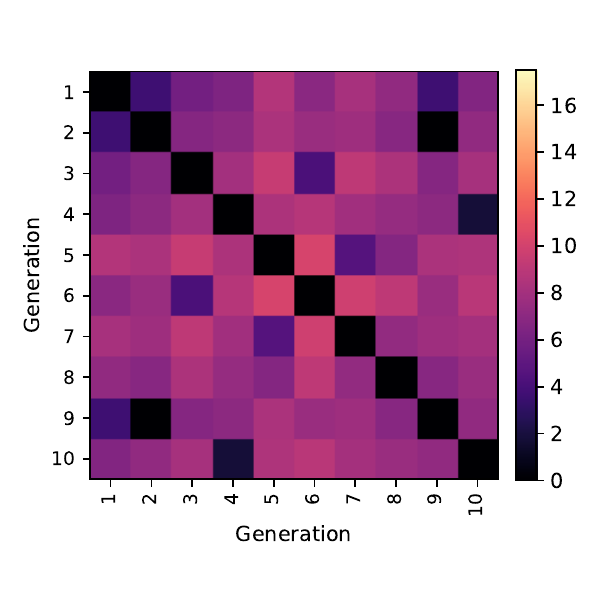}
    \caption{SQuAD (non-hallucination)}
  \end{subfigure}
  \caption{\textbf{Heatmaps of sample-to-sample OT costs on SQuAD.}
  For a prompt, $D_{ij}=W_2(\mu_i,\mu_j)$ is the Wasserstein distance between the $i$-th and $j$-th sampled responses (diagonal is zero);
  darker cells indicate larger transform costs. Hallucinated cases often exhibit larger average costs and/or more fragmented block structure,
  motivating AvgWD (magnitude) and EigenWD (structure).}
  \label{fig:ot_squad_singlecol}
  \vspace{-10pt}
\end{figure}

% -------------------------
% Visualization paragraph (extended to define Fig.4)
% -------------------------
\paragraph{Visualization.}
We visualize (i) the pairwise Wasserstein distance matrix $D\in\mathbb{R}^{K\times K}$ as a heatmap
(where $K$ is the number of sampled responses per prompt and $D_{ij}$ is the cost between the $i$-th and $j$-th responses),
and (ii) the token-level optimal transport plan $P\in\mathbb{R}_+^{n_i\times n_j}$ for a representative response pair,
where $P_{ts}$ denotes transported mass from token $t$ in response $i$ to token $s$ in response $j$
under the squared $\ell_2$ ground cost.
\textbf{Reading the heatmaps(Figs.~\ref{fig:ot_coqa_singlecol} and~\ref{fig:ot_squad_singlecol}).} Each axis indexes sampled responses for the same prompt.
Non-hallucinated cases tend to yield uniformly low costs and a single coherent neighborhood,
whereas hallucinated cases often show globally larger costs or multiple separated low-cost groups (block structure),
indicating diverse and inconsistent generations.
\textbf{Cost-graph view (Fig.~\ref{fig:ot_plan_vis}).}
To further expose the geometry induced by the sample-to-sample costs, we visualize $D$ as a weighted neighbor graph:
each node is a sampled response for a prompt, node positions are obtained from a 2D embedding computed from the precomputed distance matrix $D$
(e.g., t-SNE on distances), and edges connect nearest-neighbor pairs under $D$ (symmetrized to form an undirected graph).
Colors indicate hallucinated vs.\ non-hallucinated cases, while marker shapes (A1--A3) denote different representative prompts within each group.

% -------------------------
% Fig. 4: cost-graph visualization
% -------------------------
\begin{figure}[t]
  \centering
  \includegraphics[width=\linewidth]{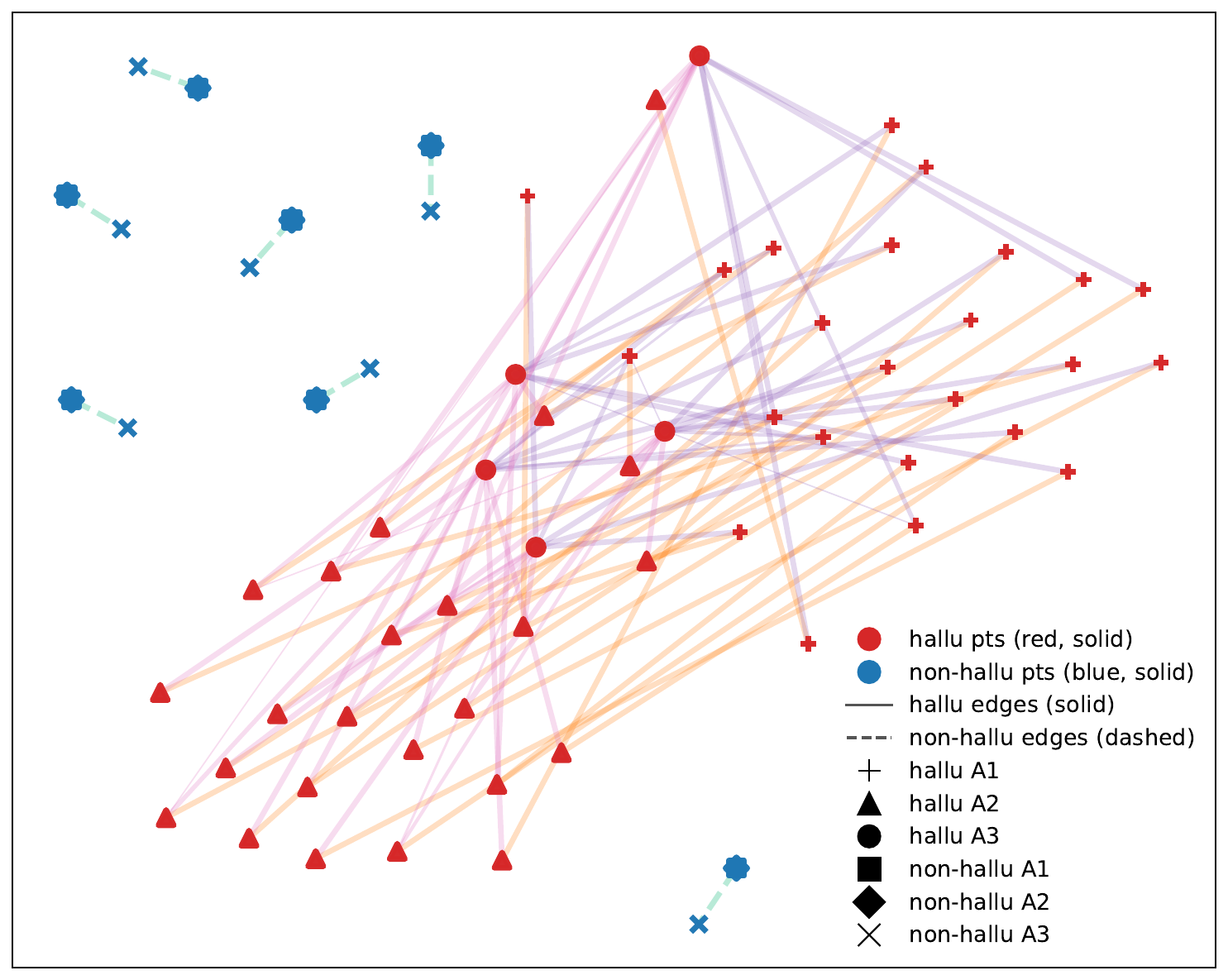}
  \caption{\textbf{Cost-graph visualization from sample transform costs.}
  Each node is a sampled response for a prompt; node positions come from a 2D embedding of the pairwise Wasserstein matrix $D$ (precomputed distances),
  and edges connect nearest-neighbor pairs under $D$ (solid: hallucinated cases; dashed: non-hallucinated).
  Marker shapes (A1--A3) denote different representative prompts within each group.}
  \label{fig:ot_plan_vis}
  \vspace{-10pt}
\end{figure}

\begin{table}[t]
\centering
\caption{White-box hallucination detection (AUROC; higher is better. Best and second-best per row are marked by \textbf{bold} and \underline{underline}.). We report both AvgWD (cost magnitude) and EigenWD (cost-structure complexity).}
\label{tab:whitebox_results}
\scriptsize
\setlength{\tabcolsep}{3pt}
\renewcommand{\arraystretch}{0.92}
\resizebox{\columnwidth}{!}{%
\begin{tabular}{llccccccc}
\toprule
Model & Dataset & AvgWD & EigenWD & ER & ES & DSE & LNE & LS \\
\midrule
\multirow{4}{*}{Llama-3.2-3b}
& CoQA & \underline{0.811} & \textbf{0.816} & 0.755 & 0.749 & 0.781 & 0.602 & 0.773 \\
& SQuAD & \underline{0.742} & \textbf{0.753} & \underline{0.742} & 0.692 & 0.713 & 0.615 & 0.741 \\
& MATH-500 & \underline{0.715} & \textbf{0.728} & 0.683 & 0.702 & 0.693 & 0.606 & 0.688 \\
& CNN/DailyMail & \underline{0.658} & \textbf{0.674} & 0.652 & 0.633 & 0.628 & 0.629 & 0.656 \\
\midrule
& Average & \underline{0.732} & \textbf{0.743} & 0.708 & 0.694 & 0.704 & 0.613 & 0.715 \\
\midrule
\multirow{4}{*}{Llama-2-7b}
& CoQA & \textbf{0.815} & \underline{0.812} & 0.794 & 0.799 & 0.805 & 0.760 & 0.802 \\
& SQuAD & \textbf{0.788} & \textbf{0.788} & 0.755 & 0.755 & 0.742 & 0.693 & \underline{0.784} \\
& MATH-500 & \underline{0.702} & \textbf{0.708} & 0.658 & 0.654 & 0.689 & 0.622 & 0.671 \\
& CNN/DailyMail & \underline{0.666} & \textbf{0.668} & 0.658 & 0.646 & 0.650 & 0.637 & 0.634 \\
\midrule
& Average & \underline{0.743} & \textbf{0.744} & 0.716 & 0.714 & 0.722 & 0.678 & 0.723 \\
\midrule
\multirow{4}{*}{Llama-3.1-8b}
& CoQA & \underline{0.771} & \textbf{0.789} & 0.729 & 0.712 & 0.761 & 0.618 & 0.769 \\
& SQuAD & \textbf{0.817} & \underline{0.799} & 0.767 & 0.785 & 0.769 & 0.613 & 0.734 \\
& MATH-500 & 0.713 & \underline{0.722} & 0.705 & 0.709 & \textbf{0.726} & 0.603 & 0.711 \\
& CNN/DailyMail & \underline{0.625} & \textbf{0.629} & 0.624 & 0.620 & 0.615 & 0.548 & 0.605 \\
\midrule
& Average & \underline{0.732} & \textbf{0.735} & 0.706 & 0.707 & 0.718 & 0.596 & 0.705 \\
\midrule
\multirow{4}{*}{Qwen3-8B}
& CoQA & \textbf{0.742} & 0.734 & 0.735 & 0.733 & 0.733 & 0.557 & \underline{0.738} \\
& SQuAD & \textbf{0.656} & 0.645 & 0.626 & 0.627 & \underline{0.647} & 0.617 & \underline{0.647} \\
& MATH-500 & \underline{0.726} & 0.712 & 0.715 & 0.718 & \textbf{0.729} & 0.656 & 0.689 \\
& CNN/DailyMail & \textbf{0.679} & 0.668 & \underline{0.675} & 0.668 & 0.654 & 0.615 & 0.645 \\
\midrule
& Average & \textbf{0.701} & 0.690 & 0.688 & 0.687 & \underline{0.691} & 0.611 & 0.680 \\
\midrule
\multirow{4}{*}{Qwen3-32B}
& CoQA & \underline{0.746} & 0.727 & 0.722 & 0.725 & 0.715 & 0.616 & \textbf{0.758} \\
& SQuAD & \textbf{0.718} & \underline{0.711} & 0.700 & 0.701 & 0.696 & 0.593 & 0.668 \\
& MATH-500 & \textbf{0.731} & \underline{0.722} & 0.704 & 0.690 & 0.672 & 0.613 & 0.659 \\
& CNN/DailyMail & \textbf{0.694} & 0.681 & 0.668 & \underline{0.685} & 0.658 & 0.606 & 0.672 \\
\midrule
& Average & \textbf{0.722} & \underline{0.710} & 0.699 & 0.700 & 0.685 & 0.607 & 0.689 \\
\bottomrule
\end{tabular}}
\end{table}

\subsection{White-box Detection Results}
\label{subsec:whitebox_results}

Table~\ref{tab:whitebox_results} reports AUROC for white-box hallucination detection across five open-source LLMs and four datasets. Overall, our \textbf{sample transform cost} signals are consistently competitive with strong training-free baselines, and often achieve the best performance.

\paragraph{Overall effectiveness.}
Across the three Llama-family models, our method yields clear gains on average. In particular, \textbf{EigenWD} achieves the highest average AUROC for Llama-3.2-3B (0.743 vs.\ 0.715 best baseline), Llama-2-7B (0.744 vs.\ 0.723), and Llama-3.1-8B (0.735 vs.\ 0.718), indicating that \emph{cost-structure complexity} is a reliable indicator of hallucination under white-box access. On Qwen models, \textbf{AvgWD} is stronger: it attains the best average AUROC for Qwen3-8B (0.701) and Qwen3-32B (0.722), outperforming the best baseline averages (0.691 and 0.700 respectively). 

\paragraph{Complementary behavior of AvgWD and EigenWD.}
AvgWD and EigenWD capture different aspects of the same Wasserstein distance matrix. AvgWD summarizes the \emph{magnitude} of sample-to-sample transform costs, while EigenWD captures the \emph{complexity} of how these costs are organized across samples. Empirically, we observe a consistent complementarity: EigenWD tends to be more advantageous on Llama models and on settings where multi-modal inconsistency arises (e.g., CoQA for Llama-3.2-3B), whereas AvgWD can be more effective when the dominant signal is the overall cost scale (e.g., Qwen3-32B on multiple datasets). This suggests that hallucinations can manifest either as uniformly larger transform costs or as a more intricate, fragmented cost structure, depending on the model family and task.

\paragraph{Across-task robustness.}
The improvements persist across heterogeneous domains, including extractive QA (CoQA, SQuAD), mathematical reasoning (MATH-500), and long-form summarization (CNN/DailyMail). Since all methods are evaluated on the same sampled generations and under identical decoding settings, the gains are attributable to the proposed representation-space transform-cost characterization rather than sampling artifacts. In summary, the main results demonstrate that \textbf{distribution complexity measured through sample transform costs} provides a strong, training-free signal for hallucination detection.

\subsection{Ablation Studies}
\label{subsec:ablations}

We analyze how decoding hyperparameters affect training-free hallucination detection, focusing on the number of sampled responses $K$ and the sampling temperature $\tau$. Since our approach characterizes the conditional distribution $p_\theta(\cdot\mid x)$ through \emph{sample-to-sample transform costs}, these hyperparameters directly change the empirical distributional complexity we observe. We report detailed results for Llama-3.1-8B in Appendix~C.

\begin{figure}[t]
  \centering
  \begin{subfigure}[t]{0.49\columnwidth}
    \centering
    \includegraphics[width=\linewidth]{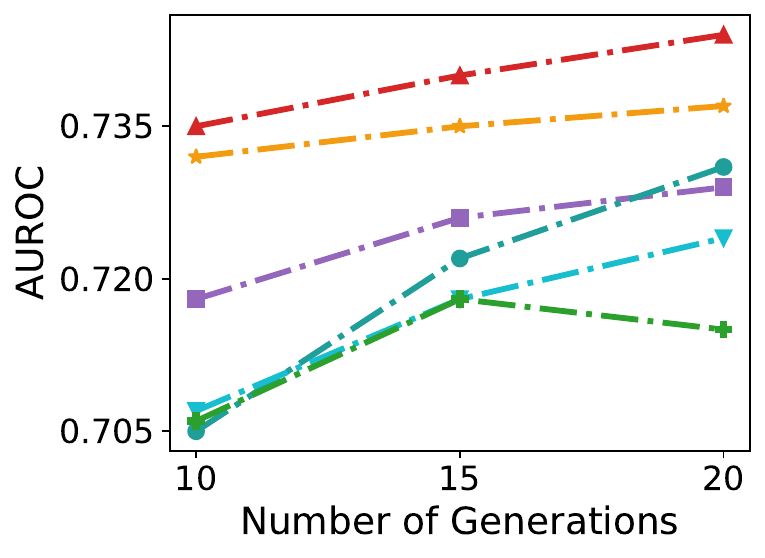}
    \caption{Varying $K$ (number of generations).}
    \label{fig:ablation_llama31_k}
  \end{subfigure}\hfill
  \begin{subfigure}[t]{0.49\columnwidth}
    \centering
    \includegraphics[width=\linewidth]{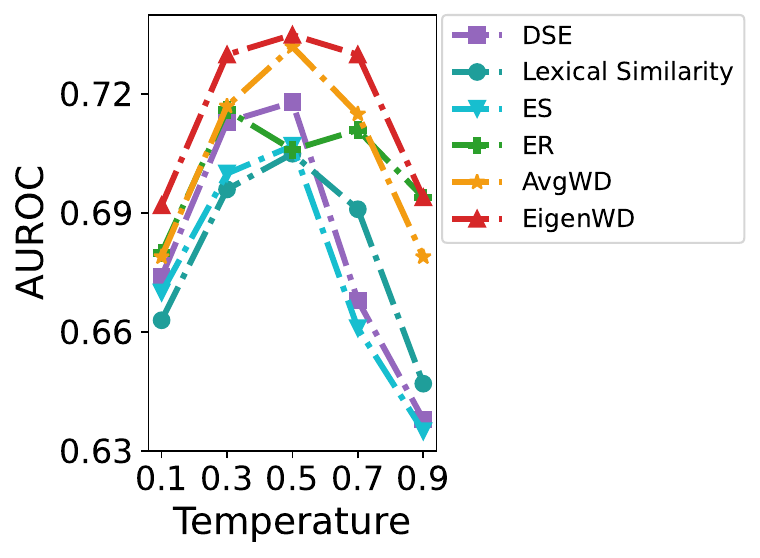}
    \caption{Varying temperature $\tau$.}
    \label{fig:ablation_llama31_tau}
  \end{subfigure}

  \vspace{-1mm}
  \caption{Ablation results on Llama-3.1-8B. Left: AUROC versus the number of sampled responses $K$. Right: AUROC versus temperature $\tau$. Unless otherwise stated, we keep decoding settings (including top-$k$ and top-$\rho$) fixed across runs.}
  \label{fig:ablation_llama31}
  \vspace{-10pt}
\end{figure}

\paragraph{Effect of the number of generations.}
We observe a consistent improvement for our OT-based signals as $K$ increases: both AvgWD (average transform cost) and EigenWD (cost-structure complexity) monotonically increase in AUROC from $K{=}10$ to $K{=}20$. This behavior is expected because larger $K$ provides a more faithful empirical characterization of the conditional distribution, yielding a more stable estimate of both the mean cost and the spectral complexity of the Wasserstein distance matrix. Notably, EigenWD benefits slightly more than AvgWD, suggesting that capturing the \emph{structure} of transform costs becomes increasingly reliable as more samples reveal multi-modal inconsistencies.

\paragraph{Effect of temperature.}
Performance is non-monotonic: moderate temperatures (around $\tau\in[0.3,0.7]$) provide the strongest detection, while very low temperature reduces diversity and very high temperature introduces excessive randomness. From the perspective of distribution complexity, $\tau$ controls a bias--variance trade-off: if $\tau$ is too small, samples collapse to near-deterministic continuations and the transform-cost matrix becomes less informative; if $\tau$ is too large, samples deviate in uncontrolled ways and the induced costs mix signal with noise. Across the full range, EigenWD remains the best-performing curve, indicating that spectral cost complexity is a robust signal even when sampling conditions change.

\subsection{Black-box Detection Case Study}
Full details and results of the black-box experiments are deferred to Appendix~\ref{app:blackbox_full} for space considerations, as the methodology and evaluation protocol remain consistent across both settings. Notably, our method also significantly outperforms all baselines in the black-box regime. For illustration, we include a representative example in Table~\ref{tab:blackbox_case} to provide a concrete case study of this behavior.

\begin{table}[t]
\centering
\caption{Black-box case study: DeepSeek-Chat (target model) with Llama-2-7B (teacher model). AUROC for hallucination detection (higher is better). Best and second-best per row are marked by \textbf{bold} and \underline{underline}.}
\label{tab:blackbox_case}
\scriptsize
\setlength{\tabcolsep}{3pt}
\renewcommand{\arraystretch}{1.0}
\resizebox{\columnwidth}{!}{%
\begin{tabular}{lcccccc}
\toprule
Dataset & \textbf{EigenWD M1} & ER & ES & LNE & LS & DSE \\
\midrule
SciQ & \textbf{0.6480} & \underline{0.6398} & 0.5784 & 0.5916 & 0.4921 & 0.5978 \\
NQ-open & \textbf{0.7861} & 0.6797 & 0.6088 & 0.6585 & 0.6804 & \underline{0.7114} \\
Math500 & \textbf{0.8889} & \underline{0.8760} & 0.5587 & 0.5344 & \underline{0.8335} & 0.5997 \\
\midrule
Average & \textbf{0.7743} & \underline{0.7318} & 0.5820 & 0.5948 & 0.6687 & 0.6363 \\
\bottomrule
\end{tabular}
}
\end{table}

\section{Conclusion}
This work revisited hallucination detection from an optimal-transport perspective and proposes a lightweight, training-free uncertainty signal derived from the geometry of hidden representations across multiple generations. Specifically, we introduce two complementary Wasserstein-based
detectors: \emph{AvgWD} and \emph{EigenWD}. Together,
they disentangle \emph{how far} responses drift in representation space from \emph{how complex} the drift pattern is, providing an interpretable view of hallucination-related uncertainty.

Extensive experiments across multiple datasets and model families demonstrate that our measures
achieve strong and consistent detection performance under common decoding settings, while remaining
efficient: they rely only on model-internal states and require no additional supervision, retrieval, or
auxiliary verification modules.

% \clearpage

\section*{Limitations}
The time complexity of our method is higher, but remains within an acceptable range.

\section*{Impact Statement}
This paper presents work whose goal is to advance the field of Machine Learning. There are many potential societal consequences of our work, none which we feel must be specifically highlighted here.

\bibliography{references}

@article{brown2020language,
  title={Language models are few-shot learners},
  author={Brown, Tom and Mann, Benjamin and Ryder, Nick and Subbiah, Melanie and Kaplan, Jared D and Dhariwal, Prafulla and Neelakantan, Arvind and Shyam, Pranav and Sastry, Girish and Askell, Amanda and others},
  journal={Advances in neural information processing systems},
  volume={33},
  pages={1877--1901},
  year={2020}
}

@article{achiam2023gpt,
  title={Gpt-4 technical report},
  author={Achiam, Josh and Adler, Steven and Agarwal, Sandhini and Ahmad, Lama and Akkaya, Ilge and Aleman, Florencia Leoni and Almeida, Diogo and Altenschmidt, Janko and Altman, Sam and Anadkat, Shyamal and others},
  journal={arXiv preprint arXiv:2303.08774},
  year={2023}
}

@article{touvron2023llama,
  title={Llama 2: Open foundation and fine-tuned chat models},
  author={Touvron, Hugo and Martin, Louis and Stone, Kevin and Albert, Peter and Almahairi, Amjad and Babaei, Yasmine and Bashlykov, Nikolay and Batra, Soumya and Bhargava, Prajjwal and Bhosale, Shruti and others},
  journal={arXiv preprint arXiv:2307.09288},
  year={2023}
}

@inproceedings{maynez2020faithfulness,
  title        = {On Faithfulness and Factuality in Abstractive Summarization},
  author       = {Maynez, Joshua and Narayan, Shashi and Bohnet, Bernd and McDonald, Ryan},
  booktitle    = {Proceedings of the 58th Annual Meeting of the Association for Computational Linguistics (ACL)},
  pages        = {1906--1919},
  year         = {2020}
}

@article{ji2023survey,
  title     = {Survey of Hallucination in Natural Language Generation},
  author    = {Ji, Ziwei and Lee, Nayeon and Frieske, Rita and Yu, Tiezheng and Su, Dan and Xu, Yan and Ishii, Etsuko and Bang, Yejin and Chen, Delong and Dai, Wenliang and Chan, Ho Shu and Madotto, Andrea and Fung, Pascale},
  journal   = {ACM Computing Surveys},
  volume    = {55},
  number    = {12},
  pages     = {1--38},
  year      = {2023},
  publisher = {ACM}
}

@article{huang2025survey,
  title={A survey on hallucination in large language models: Principles, taxonomy, challenges, and open questions},
  author={Huang, Lei and Yu, Weijiang and Ma, Weitao and Zhong, Weihong and Feng, Zhangyin and Wang, Haotian and Chen, Qianglong and Peng, Weihua and Feng, Xiaocheng and Qin, Bing and others},
  journal={ACM Transactions on Information Systems},
  volume={43},
  number={2},
  pages={1--55},
  year={2025},
  publisher={ACM New York, NY}
}

@inproceedings{lin2022truthfulqa,
  title={Truthfulqa: Measuring how models mimic human falsehoods},
  author={Lin, Stephanie and Hilton, Jacob and Evans, Owain},
  booktitle={Proceedings of the 60th annual meeting of the association for computational linguistics (volume 1: long papers)},
  pages={3214--3252},
  year={2022}
}

@article{welleck2019neural,
  title={Neural text generation with unlikelihood training},
  author={Welleck, Sean and Kulikov, Ilia and Roller, Stephen and Dinan, Emily and Cho, Kyunghyun and Weston, Jason},
  journal={arXiv preprint arXiv:1908.04319},
  year={2019}
}

@article{kendall2017uncertainties,
  title={What uncertainties do we need in bayesian deep learning for computer vision?},
  author={Kendall, Alex and Gal, Yarin},
  journal={Advances in neural information processing systems},
  volume={30},
  year={2017}
}

@inproceedings{gal2016dropout,
  title={Dropout as a bayesian approximation: Representing model uncertainty in deep learning},
  author={Gal, Yarin and Ghahramani, Zoubin},
  booktitle={international conference on machine learning},
  pages={1050--1059},
  year={2016},
  organization={PMLR}
}

@article{lakshminarayanan2017simple,
  title={Simple and scalable predictive uncertainty estimation using deep ensembles},
  author={Lakshminarayanan, Balaji and Pritzel, Alexander and Blundell, Charles},
  journal={Advances in neural information processing systems},
  volume={30},
  year={2017}
}

@article{farquhar2024detecting,
  title={Detecting hallucinations in large language models using semantic entropy},
  author={Farquhar, Sebastian and Kossen, Jannik and Kuhn, Lorenz and Gal, Yarin},
  journal={Nature},
  volume={630},
  number={8017},
  pages={625--630},
  year={2024},
  publisher={Nature Publishing Group UK London}
}

@article{vaswani2017attention,
  title={Attention is all you need},
  author={Vaswani, Ashish and Shazeer, Noam and Parmar, Niki and Uszkoreit, Jakob and Jones, Llion and Gomez, Aidan N and Kaiser, {\L}ukasz and Polosukhin, Illia},
  journal={Advances in neural information processing systems},
  volume={30},
  year={2017}
}

@article{cuturi2013sinkhorn,
  title={Sinkhorn distances: Lightspeed computation of optimal transport},
  author={Cuturi, Marco},
  journal={Advances in neural information processing systems},
  volume={26},
  year={2013}
}

@inproceedings{roy2007effective,
  title={The effective rank: A measure of effective dimensionality},
  author={Roy, Olivier and Vetterli, Martin},
  booktitle={2007 15th European signal processing conference},
  pages={606--610},
  year={2007},
  organization={IEEE}
}

@inproceedings{manakul2023selfcheckgpt,
  title={Selfcheckgpt: Zero-resource black-box hallucination detection for generative large language models},
  author={Manakul, Potsawee and Liusie, Adian and Gales, Mark},
  booktitle={Proceedings of the 2023 conference on empirical methods in natural language processing},
  pages={9004--9017},
  year={2023}
}

@article{lewis2020retrieval,
  title={Retrieval-augmented generation for knowledge-intensive nlp tasks},
  author={Lewis, Patrick and Perez, Ethan and Piktus, Aleksandra and Petroni, Fabio and Karpukhin, Vladimir and Goyal, Naman and K{\"u}ttler, Heinrich and Lewis, Mike and Yih, Wen-tau and Rockt{\"a}schel, Tim and others},
  journal={Advances in neural information processing systems},
  volume={33},
  pages={9459--9474},
  year={2020}
}

@article{kuhn2023semantic,
  title={Semantic uncertainty: Linguistic invariances for uncertainty estimation in natural language generation},
  author={Kuhn, Lorenz and Gal, Yarin and Farquhar, Sebastian},
  journal={arXiv preprint arXiv:2302.09664},
  year={2023}
}

@article{kadavath2022language,
  title={Language models (mostly) know what they know},
  author={Kadavath, Saurav and Conerly, Tom and Askell, Amanda and Henighan, Tom and Drain, Dawn and Perez, Ethan and Schiefer, Nicholas and Hatfield-Dodds, Zac and DasSarma, Nova and Tran-Johnson, Eli and others},
  journal={arXiv preprint arXiv:2207.05221},
  year={2022}
}

@article{chen2024inside,
  title={INSIDE: LLMs' internal states retain the power of hallucination detection},
  author={Chen, Chao and Liu, Kai and Chen, Ze and Gu, Yi and Wu, Yue and Tao, Mingyuan and Fu, Zhihang and Ye, Jieping},
  journal={arXiv preprint arXiv:2402.03744},
  year={2024}
}

@article{malinin2020uncertainty,
  title={Uncertainty estimation in autoregressive structured prediction},
  author={Malinin, Andrey and Gales, Mark},
  journal={arXiv preprint arXiv:2002.07650},
  year={2020}
}

@article{sriramanan2024llm,
  title={Llm-check: Investigating detection of hallucinations in large language models},
  author={Sriramanan, Gaurang and Bharti, Siddhant and Sadasivan, Vinu Sankar and Saha, Shoumik and Kattakinda, Priyatham and Feizi, Soheil},
  journal={Advances in Neural Information Processing Systems},
  volume={37},
  pages={34188--34216},
  year={2024}
}

@book{peyre2019computational,
  title={Computational optimal transport: With applications to data science},
  author={Peyr{\'e}, Gabriel and Cuturi, Marco},
  year={2019},
  publisher={Now Foundations and Trends}
}

@article{wang2025revisiting,
  title={Revisiting Hallucination Detection with Effective Rank-based Uncertainty},
  author={Wang, Rui and Wei, Zeming and Yue, Guanzhang and Sun, Meng},
  journal={arXiv preprint arXiv:2510.08389},
  year={2025}
}

@inproceedings{li2023halueval,
  title={Halueval: A large-scale hallucination evaluation benchmark for large language models},
  author={Li, Junyi and Cheng, Xiaoxue and Zhao, Wayne Xin and Nie, Jian-Yun and Wen, Ji-Rong},
  booktitle={Proceedings of the 2023 conference on empirical methods in natural language processing},
  pages={6449--6464},
  year={2023}
}

@inproceedings{gao2023rarr,
  title={Rarr: Researching and revising what language models say, using language models},
  author={Gao, Luyu and Dai, Zhuyun and Pasupat, Panupong and Chen, Anthony and Chaganty, Arun Tejasvi and Fan, Yicheng and Zhao, Vincent and Lao, Ni and Lee, Hongrae and Juan, Da-Cheng and others},
  booktitle={Proceedings of the 61st Annual Meeting of the Association for Computational Linguistics (Volume 1: Long Papers)},
  pages={16477--16508},
  year={2023}
}

@inproceedings{liu2023g,
  title={G-eval: NLG evaluation using gpt-4 with better human alignment},
  author={Liu, Yang and Iter, Dan and Xu, Yichong and Wang, Shuohang and Xu, Ruochen and Zhu, Chenguang},
  booktitle={Proceedings of the 2023 conference on empirical methods in natural language processing},
  pages={2511--2522},
  year={2023}
}

@inproceedings{min2023factscore,
  title={Factscore: Fine-grained atomic evaluation of factual precision in long form text generation},
  author={Min, Sewon and Krishna, Kalpesh and Lyu, Xinxi and Lewis, Mike and Yih, Wen-tau and Koh, Pang and Iyyer, Mohit and Zettlemoyer, Luke and Hajishirzi, Hannaneh},
  booktitle={Proceedings of the 2023 Conference on Empirical Methods in Natural Language Processing},
  pages={12076--12100},
  year={2023}
}

@article{reddy2019coqa,
  title={Coqa: A conversational question answering challenge},
  author={Reddy, Siva and Chen, Danqi and Manning, Christopher D},
  journal={Transactions of the Association for Computational Linguistics},
  volume={7},
  pages={249--266},
  year={2019},
  publisher={MIT Press One Rogers Street, Cambridge, MA 02142-1209, USA journals-info~…}
}

@inproceedings{rajpurkar2016squad,
  title={Squad: 100,000+ questions for machine comprehension of text},
  author={Rajpurkar, Pranav and Zhang, Jian and Lopyrev, Konstantin and Liang, Percy},
  booktitle={Proceedings of the 2016 conference on empirical methods in natural language processing},
  pages={2383--2392},
  year={2016}
}

@article{hendrycks2021measuring,
  title={Measuring mathematical problem solving with the math dataset},
  author={Hendrycks, Dan and Burns, Collin and Kadavath, Saurav and Arora, Akul and Basart, Steven and Tang, Eric and Song, Dawn and Steinhardt, Jacob},
  journal={arXiv preprint arXiv:2103.03874},
  year={2021}
}

@article{hermann2015teaching,
  title={Teaching machines to read and comprehend},
  author={Hermann, Karl Moritz and Kocisky, Tomas and Grefenstette, Edward and Espeholt, Lasse and Kay, Will and Suleyman, Mustafa and Blunsom, Phil},
  journal={Advances in neural information processing systems},
  volume={28},
  year={2015}
}

@inproceedings{nallapati2016abstractive,
  title={Abstractive text summarization using sequence-to-sequence rnns and beyond},
  author={Nallapati, Ramesh and Zhou, Bowen and Dos Santos, Cicero and Gul{\c{c}}ehre, {\c{C}}a{\u{g}}lar and Xiang, Bing},
  booktitle={Proceedings of the 20th SIGNLL conference on computational natural language learning},
  pages={280--290},
  year={2016}
}

@inproceedings{lin2004rouge,
  title={Rouge: A package for automatic evaluation of summaries},
  author={Lin, Chin-Yew},
  booktitle={Text summarization branches out},
  pages={74--81},
  year={2004}
}
\bibliographystyle{icml2026}

\appendix

\section{Black-box Experiment and Full Results}
\label{app:blackbox_full}
We employed two representative black-box models, GPT-4o-mini and DeepSeek-chat, alongside three typical medium-sized open-source models as white-box teachers: Llama-2-7b, Llama-3.1-8b, and Qwen-3-8b. As enterprise-grade models with strong capabilities, GPT-4o-mini and DeepSeek-chat exhibit low hallucination rates on simpler datasets such as CoQA and SQuAD, making these datasets inadequate for effective evaluation. Therefore, in our black-box experiments, we selected three more challenging datasets that better reflect real-world user queries: SciQ, NQ-open, and Math500.

It is worth noting that for Math500, where reference answers are concise numerical values while model outputs contain complex reasoning chains, we employed an LLM-as-Judge framework using GPT-4o-mini for hallucination labeling instead of using ROUGE and RoBERTa. Following the teacher-forcing approach described in the main text, we conducted three sets of experiments using 10, 15, and 20 responses, with results shown in Tables~\ref{tab:blackbox_full}, \ref{tab:blackbox_full_v2}, and \ref{tab:blackbox_full_v3}, respectively.

Our experimental results demonstrate that this approach achieves competitive AUROC scores across all configurations, ranking first in the majority of dataset-model combinations and consistently placing in the top-2 for nearly all cases, validating its effectiveness and superiority over baseline methods in this black-box context. These results further establish the practical value and industrial applicability of our method.

\begin{table*}[t]
\centering
\caption{Black-box hallucination detection with teacher forcing, 10 samples (AUROC; higher is better). Best and second-best per row are marked by \textbf{bold} and \underline{underline}.}
\label{tab:blackbox_full}
\scriptsize
\setlength{\tabcolsep}{3pt}
\renewcommand{\arraystretch}{1.0}

\resizebox{\textwidth}{!}{%
\begin{tabular}{lllcccccc}
\toprule
Black model & Teacher model & Dataset & \textbf{EigenWD M1} & ER & ES & LNE & LS & DSE \\
\midrule

\multirow{9}{*}{chatgpt-4o-mini}
& \multirow{3}{*}{Llama-2-7b}
& SciQ
& 0.6032
& \textbf{0.6382}
& \underline{0.6252}
& 0.6223
& 0.6145
& 0.6160 \\

& & NQ-open
& \textbf{0.6414}
& 0.6350
& 0.5680
& \underline{0.6385}
& 0.6276
& 0.5803 \\

& & Math500
& \underline{0.7876}
& \textbf{0.8099}
& 0.6307
& 0.6449
& \underline{0.7795}
& 0.5686 \\
\noalign{\vskip 2.0pt}
\cline{2-9}
\noalign{\vskip 2.0pt}

& \multirow{3}{*}{Llama-3.1-8b}
& SciQ
& \textbf{0.6251}
& 0.6108
& 0.6148
& 0.6038
& 0.6145
& \underline{0.6160} \\

& & NQ-open
& \textbf{0.6303}
& 0.6263
& 0.6205
& 0.6207
& \underline{0.6276}
& 0.5803 \\

& & Math500
& \textbf{0.8329}
& 0.7556
& 0.6794
& 0.6176
& \underline{0.7795}
& 0.5686 \\
\noalign{\vskip 2.0pt}
\cline{2-9}
\noalign{\vskip 2.0pt}

& \multirow{3}{*}{Qwen3-8B}
& SciQ
& \underline{0.6228}
& 0.6196
& 0.6154
& \textbf{0.6537}
& 0.6145
& 0.6160 \\

& & NQ-open
& \textbf{0.6613}
& \underline{0.6284}
& 0.6251
& 0.6240
& 0.6276
& 0.5803 \\

& & Math500
& 0.6803
& \textbf{0.7809}
& 0.6207
& 0.6595
& \underline{0.7795}
& 0.5686 \\
\midrule
\multirow{9}{*}{deepseek-chat}
& \multirow{3}{*}{Llama-2-7b}
& SciQ
& \textbf{0.6480}
& \underline{0.6398}
& 0.5784
& 0.5916
& 0.4921
& 0.5978 \\

& & NQ-open
& \textbf{0.7861}
& 0.6797
& 0.6088
& 0.6585
& 0.6804
& \underline{0.7114} \\

& & Math500
& \textbf{0.8889}
& \underline{0.8760}
& 0.5587
& 0.5344
& \underline{0.8335}
& 0.5997 \\
\noalign{\vskip 2.0pt}
\cline{2-9}
\noalign{\vskip 2.0pt}

& \multirow{3}{*}{Llama-3.1-8b}
& SciQ
& \textbf{0.6392}
& 0.5988
& \underline{0.6314}
& 0.5782
& 0.4921
& 0.5978 \\

& & NQ-open
& \underline{0.7648}
& 0.7366
& \textbf{0.7670}
& 0.5454
& 0.6804
& 0.7114 \\

& & Math500
& \textbf{0.8738}
& 0.8174
& 0.6381
& 0.5470
& \underline{0.8335}
& 0.5997 \\
\noalign{\vskip 2.0pt}
\cline{2-9}
\noalign{\vskip 2.0pt}

& \multirow{3}{*}{Qwen3-8B}
& SciQ
& \underline{0.6102}
& 0.6090
& \textbf{0.6534}
& 0.5927
& 0.4921
& 0.5978 \\

& & NQ-open
& \textbf{0.7675}
& 0.7427
& \underline{0.7583}
& 0.5083
& 0.6804
& 0.7114\\

& & Math500
& 0.7465
& \underline{0.8014}
& 0.5803
& 0.4952
& \textbf{0.8335}
& 0.5997 \\
\bottomrule
\end{tabular}
}
\end{table*}

\begin{table*}[t]
\centering
\caption{Black-box hallucination detection with teacher forcing, 15 samples (AUROC; higher is better). Best and second-best per row are marked by \textbf{bold} and \underline{underline}.}
\label{tab:blackbox_full_v2}
\scriptsize
\setlength{\tabcolsep}{3pt}
\renewcommand{\arraystretch}{1.00}

\resizebox{\textwidth}{!}{%
\begin{tabular}{lllcccccc}
\toprule
Black model & Teacher model & Dataset & \textbf{EigenWD M1} & ER & ES & LNE & LS & DSE \\
\midrule

\multirow{9}{*}{chatgpt-4o-mini}
& \multirow{3}{*}{Llama-2-7b}
& SciQ
& \underline{0.6302}
& \textbf{0.6409}
& 0.6074
& 0.6154
& 0.6075
& 0.6172 \\

& & NQ-open
& \textbf{0.6368}
& \underline{0.6338}
& 0.5677
& 0.6024
& 0.6199
& 0.6097 \\

& & Math500
& \textbf{0.8652}
& \underline{0.8049}
& 0.6534
& 0.6389
& 0.7723
& 0.6071 \\
\noalign{\vskip 2.0pt}
\cline{2-9}
\noalign{\vskip 2.0pt}

& \multirow{3}{*}{Llama-3.1-8b}
& SciQ
& 0.6280
& \underline{0.6330}
& \textbf{0.6406}
& 0.5862
& 0.6075
& 0.6172 \\

& & NQ-open
& \textbf{0.6344}
& 0.6300
& 0.6238
& 0.6060
& 0.6199
& 0.6097 \\

& & Math500
& \textbf{0.8387}
& 0.7660
& 0.6993
& 0.6306
& 0.7723
& 0.6071 \\
\noalign{\vskip 2.0pt}
\cline{2-9}
\noalign{\vskip 2.0pt}

& \multirow{3}{*}{Qwen3-8B}
& SciQ
& \textbf{0.6444}
& 0.6386
& \underline{0.6432}
& 0.6350
& 0.6075
& 0.6172 \\

& & NQ-open
& \textbf{0.6557}
& \underline{0.6335}
& 0.6286
& 0.5838
& 0.6199
& 0.6097 \\

& & Math500
& 0.6868
& \textbf{0.7736}
& 0.6600
& 0.6592
& \underline{0.7723}
& 0.6071 \\
\midrule
\multirow{9}{*}{deepseek-chat}
& \multirow{3}{*}{Llama-2-7b}
& SciQ
& \textbf{0.6773}
& \underline{0.6389}
& 0.5586
& 0.5769
& 0.4605
& 0.6285 \\

& & NQ-open
& \textbf{0.7837}
& 0.6895
& 0.5004
& 0.6335
& 0.6932
& \underline{0.7264} \\

& & Math500
& \underline{0.8750}
& \textbf{0.8919}
& 0.6492
& 0.5806
& 0.8109
& 0.5689 \\
\noalign{\vskip 2.0pt}
\cline{2-9}
\noalign{\vskip 2.0pt}

& \multirow{3}{*}{Llama-3.1-8b}
& SciQ
& \textbf{0.6588}
& 0.6046
& \underline{0.6445}
& 0.5494
& 0.4605
& 0.6285 \\

& & NQ-open
& \textbf{0.7921}
& 0.7456
& \underline{0.7603}
& 0.5235
& 0.6932
& \underline{0.7264} \\

& & Math500
& \textbf{0.8314}
& \underline{0.8216}
& 0.6987
& 0.5710
& 0.8109
& 0.5689 \\
\noalign{\vskip 2.0pt}
\cline{2-9}
\noalign{\vskip 2.0pt}

& \multirow{3}{*}{Qwen3-8B}
& SciQ
& \textbf{0.6479}
& 0.5835
& \underline{0.6435}
& 0.5767
& 0.4605
& 0.6285 \\

& & NQ-open
& \textbf{0.7904}
& 0.7481
& \underline{0.7614}
& 0.5110
& 0.6932
& 0.7264 \\

& & Math500
& 0.7436
& \underline{0.8600}
& 0.6820
& 0.5077
& \textbf{0.8109}
& 0.5689 \\
\bottomrule
\end{tabular}
}
\end{table*}

\begin{table*}[t]
\centering
\caption{Black-box hallucination detection with teacher forcing, 20 samples (AUROC; higher is better). Best and second-best per row are marked by \textbf{bold} and \underline{underline}.}
\label{tab:blackbox_full_v3}
\scriptsize
\setlength{\tabcolsep}{3pt}
\renewcommand{\arraystretch}{1.00}

\resizebox{\textwidth}{!}{%
\begin{tabular}{lllcccccc}
\toprule
Black model & Teacher model & Dataset & \textbf{EigenWD M1} & ER & ES & LNE & LS & DSE \\
\midrule

\multirow{9}{*}{chatgpt-4o-mini}
& \multirow{3}{*}{Llama-2-7b}
& SciQ
& \underline{0.6192}
& \textbf{0.6412}
& 0.6073
& 0.6092
& 0.6134
& 0.6167 \\

& & NQ-open
& \textbf{0.6465}
& \underline{0.6321}
& 0.5520
& 0.6043
& 0.6323
& 0.6130 \\

& & Math500
& \textbf{0.8567}
& \underline{0.8035}
& 0.6540
& 0.6648
& \underline{0.7911}
& 0.6010 \\
\noalign{\vskip 2.0pt}
\cline{2-9}
\noalign{\vskip 2.0pt}

& \multirow{3}{*}{Llama-3.1-8b}
& SciQ
& 0.6204
& \underline{0.6429}
& \textbf{0.6481}
& 0.5876
& 0.6134
& 0.6167 \\

& & NQ-open
& \textbf{0.6356}
& \underline{0.6338}
& 0.6250
& 0.6084
& 0.6323
& 0.6130 \\

& & Math500
& \textbf{0.8645}
& 0.7548
& 0.6841
& 0.6057
& \underline{0.7911}
& 0.6010 \\
\noalign{\vskip 2.0pt}
\cline{2-9}
\noalign{\vskip 2.0pt}

& \multirow{3}{*}{Qwen3-8B}
& SciQ
& 0.6199
& \underline{0.6434}
& \textbf{0.6502}
& 0.6398
& 0.6134
& 0.6167 \\

& & NQ-open
& \textbf{0.6401}
& 0.6300
& \underline{0.6326}
& 0.5884
& 0.6323
& 0.6130 \\

& & Math500
& 0.7281
& \textbf{0.7968}
& 0.6822
& 0.6688
& \underline{0.7911}
& 0.6010 \\
\midrule
\multirow{9}{*}{deepseek-chat}
& \multirow{3}{*}{Llama-2-7b}
& SciQ
& \textbf{0.6823}
& \underline{0.6442}
& 0.5413
& 0.5759
& 0.4744
& 0.6408 \\

& & NQ-open
& \textbf{0.7758}
& 0.6818
& 0.5423
& 0.6342
& 0.6998
& \underline{0.7285} \\

& & Math500
& \underline{0.8544}
& \textbf{0.9006}
& 0.6888
& 0.5685
& 0.8491
& 0.5754 \\
\noalign{\vskip 2.0pt}
\cline{2-9}
\noalign{\vskip 2.0pt}

& \multirow{3}{*}{Llama-3.1-8b}
& SciQ
& \textbf{0.6647}
& 0.6152
& \underline{0.6456}
& 0.5500
& 0.4744
& 0.6408 \\

& & NQ-open
& \textbf{0.7904}
& 0.7422
& \underline{0.7556}
& 0.5291
& 0.6998
& \underline{0.7285} \\

& & Math500
& \textbf{0.8501}
& 0.8442
& 0.7216
& 0.5829
& \underline{0.8491}
& 0.5754 \\
\noalign{\vskip 2.0pt}
\cline{2-9}
\noalign{\vskip 2.0pt}

& \multirow{3}{*}{Qwen3-8B}
& SciQ
& \underline{0.6442}
& 0.6016
& \textbf{0.6528}
& 0.5802
& 0.4744
& 0.6408 \\

& & NQ-open
& \textbf{0.7878}
& 0.7454
& \underline{0.7601}
& 0.5196
& 0.6998
& 0.7285 \\

& & Math500
& 0.7348
& \textbf{0.8672}
& 0.6975
& 0.5084
& \underline{0.8491}
& 0.5754 \\
\bottomrule
\end{tabular}
}
\end{table*}

\section{Proofs for Robustness Lemmas}
\label{app:stability_proofs}

\subsection{Proof of Lemma 1}
\begin{proof}
By the triangle inequality for the Wasserstein distance,
\begin{align}
W_2(\mu,\nu)
&\le W_2(\mu,\mu') + W_2(\mu',\nu') + W_2(\nu',\nu), \\
W_2(\mu',\nu')
&\le W_2(\mu',\mu) + W_2(\mu,\nu) + W_2(\nu,\nu').
\end{align}
Rearranging the two inequalities gives
\begin{equation}
\bigl|W_2(\mu,\nu)-W_2(\mu',\nu')\bigr|
\le
W_2(\mu,\mu') + W_2(\nu,\nu').
\end{equation}
\end{proof}

\subsection{Proof of Lemma 2}
\begin{proof}
Let $m$ be the number of tokens. Consider the coupling
\begin{equation}
\pi \;=\; \frac{1}{m}\sum_{t=1}^{m}\delta_{(\mathbf{z}_t,\mathbf{z}'_t)},
\end{equation}
which matches each token $\mathbf{z}_t$ to $\mathbf{z}'_t$ with mass $1/m$.
This is a valid transport plan between the uniform empirical measures
$\mu(\mathbf{Z})$ and $\mu(\mathbf{Z}')$. Therefore,
\begin{align}
W_2^2\!\bigl(\mu(\mathbf{Z}),\mu(\mathbf{Z}')\bigr)
&\le \int \|u-v\|_2^2 \, d\pi(u,v) \notag\\
&= \frac{1}{m}\sum_{t=1}^{m}\|\mathbf{z}_t-\mathbf{z}'_t\|_2^2 \notag\\
&= \frac{\|\mathbf{Z}-\mathbf{Z}'\|_F^2}{m}.
\label{eq:token_w2_proof_chain}
\end{align}
Taking square roots yields \Cref{eq:token_w2}.
\end{proof}

\subsection{Proof of Theorem 1}
\begin{proof}
By \Cref{lem:two_sided_w2} with $(\mu,\nu)=(\mu_i,\mu_j)$ and $(\mu',\nu')=(\mu_i',\mu_j')$,
\begin{equation}
|D_{ij}-D'_{ij}|
\le
W_2(\mu_i,\mu_i') + W_2(\mu_j,\mu_j').
\end{equation}
By \Cref{lem:token_w2},
\begin{equation}
W_2(\mu_i,\mu_i')\le \varepsilon_i,
\qquad
W_2(\mu_j,\mu_j')\le \varepsilon_j,
\end{equation}
hence
\begin{equation}
|D_{ij}-D'_{ij}|\le \varepsilon_i+\varepsilon_j.
\end{equation}
Averaging over all unordered pairs and using the definition of AvgWD, together with
$\sum_{1\le i<j\le K}(\varepsilon_i+\varepsilon_j)=(K-1)\sum_{i=1}^K \varepsilon_i$,
gives
\begin{equation}
\bigl|\mathrm{AvgWD}(\mathbf{Z}_{1:K})-\mathrm{AvgWD}(\mathbf{Z}'_{1:K})\bigr|
\le \frac{2}{K}\sum_{i=1}^K \varepsilon_i.
\end{equation}
\end{proof}

\subsection{A stability statement for EigenWD (optional)}
For completeness, we record a Lipschitz-type bound for EigenWD under boundedness assumptions. Let $\mathbf{D},\mathbf{D}'\in\mathbb{R}^{K\times K}$ be two symmetric distance matrices with $D_{ii}=D'_{ii}=0$, and define
\begin{equation}
K(\mathbf{D})_{ij}=\exp\!\left(-\frac{D_{ij}^2}{2(b^2+\epsilon)}\right)+\alpha\,\mathbf{1}[i=j],
\end{equation}
with fixed $b,\epsilon,\alpha>0$.

\begin{lemma}[EigenWD is locally Lipschitz in $\mathbf{D}$]
\label{lem:eigenwd_lipschitz}
Assume $\max_{i,j}\{D_{ij},D'_{ij}\}\le R$ for some $R>0$, and let $\boldsymbol{\lambda}(\mathbf{D})$ denote the eigenvalues of $K(\mathbf{D})$. Define
\begin{equation}
\mathrm{EigenWD}(\mathbf{D})=\frac{\|\boldsymbol{\lambda}(\mathbf{D})\|_p}{\|\boldsymbol{\lambda}(\mathbf{D})\|_2}
\qquad \text{for } p\in(0,2).
\end{equation}
Then there exists a constant $C=C(K,p,\alpha,b,\epsilon,R)$ such that
\begin{equation}
\bigl|\mathrm{EigenWD}(\mathbf{D})-\mathrm{EigenWD}(\mathbf{D}')\bigr|
\;\le\;
C\,\|\mathbf{D}-\mathbf{D}'\|_F.
\end{equation}
\end{lemma}

\begin{proof}
First, define the entrywise map
\begin{equation}
g(x) \;=\; \exp\!\left(-\frac{x^2}{2(b^2+\epsilon)}\right).
\end{equation}
It is Lipschitz on $[0,R]$ with constant
\begin{align}
L_g
&= \max_{x\in[0,R]}\left|\frac{d}{dx}g(x)\right|
= \max_{x\in[0,R]}\frac{x}{b^2+\epsilon}\exp\!\left(-\frac{x^2}{2(b^2+\epsilon)}\right)\notag\\
&\le \frac{R}{b^2+\epsilon}.
\end{align}
Hence,
\begin{equation}
\|K(\mathbf{D})-K(\mathbf{D}')\|_F \le L_g\,\|\mathbf{D}-\mathbf{D}'\|_F.
\end{equation}

Second, by Hoffman--Wielandt inequality for symmetric matrices,
\begin{equation}
\|\boldsymbol{\lambda}(\mathbf{D})-\boldsymbol{\lambda}(\mathbf{D}')\|_2
\le \|K(\mathbf{D})-K(\mathbf{D}')\|_F
\le L_g\,\|\mathbf{D}-\mathbf{D}'\|_F.
\end{equation}

Third, the diagonal shift implies $K(\mathbf{D})\succeq \alpha I$, hence
$\|\boldsymbol{\lambda}(\mathbf{D})\|_2 \ge \sqrt{K}\alpha$ (and similarly for $\mathbf{D}'$).

Finally, for $p\in(0,2)$, the map
\begin{equation}
f(\boldsymbol{\lambda})=\frac{\|\boldsymbol{\lambda}\|_p}{\|\boldsymbol{\lambda}\|_2}
\end{equation}
is locally Lipschitz on the compact set induced by the above bounds, so
\begin{align}
|f(\boldsymbol{\lambda}(\mathbf{D}))-f(\boldsymbol{\lambda}(\mathbf{D}'))|
&\le C\,\|\boldsymbol{\lambda}(\mathbf{D})-\boldsymbol{\lambda}(\mathbf{D}')\|_2 \notag\\
&\le C' L_g\,\|\mathbf{D}-\mathbf{D}'\|_F.
\label{eq:eigenwd_lip_final}
\end{align}
This completes the proof.
\end{proof}

% =========================
% Appendix C: Ablation (t vs N separated)
% =========================
% -------------------------
\section{Additional Ablation Results}
\label{app:ablation}

\subsection{Ablation on Sampling Hyperparameters}
\label{app:ablation_sampling}
We provide full ablation tables for (i) sampling temperature $\tau$ and (ii) the number of stochastic generations $K$.
These hyperparameters directly control the diversity of samples drawn from $p_\theta(\cdot\mid x)$ and therefore affect the observed \emph{sample transform costs} used by AvgWD/EigenWD.
Unless otherwise stated, we keep the decoding configuration fixed (including top-$k$ / top-$\rho$ truncation) and only vary the target factor.
For each setting we report AUROC on each dataset and the mean across datasets (row ``Average'').

\subsection{Qwen3-8B}
\label{app:qwen_ablation}
Tables~\ref{tab:appendixC_qwen_t} and~\ref{tab:appendixC_qwen_n} report ablations on Qwen3-8B.
Overall, AvgWD tends to be stronger than EigenWD on Qwen models (consistent with the main results), indicating that the dominant signal is often the \emph{magnitude} of sample transform costs rather than a highly fragmented cost structure.

\paragraph{Temperature $\tau$.}
As shown in Table~\ref{tab:appendixC_qwen_t}, detection performance is generally non-monotonic in $\tau$.
Very low temperature reduces sample diversity and can weaken the transform-cost signal, while overly high temperature introduces excessive randomness that dilutes factual inconsistency cues.
In our experiments, moderate temperatures (e.g., $\tau\in[0.5,0.7]$) typically yield the best average performance across datasets, aligning with the intuition that an informative estimate of distribution complexity requires neither near-deterministic nor overly noisy sampling.

\paragraph{Number of generations $K$.}
Table~\ref{tab:appendixC_qwen_n} presents experiment results with $K \in \{10, 15, 20\}$.
Increasing $K$ provides a richer empirical characterization of $p_\theta(\cdot\mid x)$ and can improve stability of both AvgWD and EigenWD.
On Qwen3-8B, the gains with larger $K$ are present but moderate on average, suggesting that Qwen’s sample transform costs are already relatively informative at $K{=}10$ while additional samples mainly refine the estimate.

\paragraph{Practical recommendation.}
For Qwen3-8B, we recommend $\tau\approx 0.5$--$0.7$ with $K\ge 10$ as a good default trade-off between detection performance and inference cost.

\begin{table*}[t]
\centering
\caption{Ablation on sampling temperature $t$ for Qwen3-8B (AUROC; higher is better). Best and second-best per row are marked by \textbf{bold} and \underline{underline}.}
\label{tab:appendixC_qwen_t}
\scriptsize
\setlength{\tabcolsep}{3pt}
\renewcommand{\arraystretch}{0.95}
\resizebox{\textwidth}{!}{%
\begin{tabular}{llccccccccc}
\toprule
$t$ & Dataset & AvgWD M1 & AvgWD L1 & EigenWD M1 & EigenWD L1 & ER & ES & DSE & LNE & LS \\
\midrule
\multirow{5}{*}{0.1} & CoQA & 0.642 & 0.633 & \textbf{0.645} & \underline{0.643} & 0.630 & 0.620 & 0.635 & 0.596 & 0.631 \\
 & SQuAD & \textbf{0.577} & 0.561 & \textbf{0.577} & \textbf{0.577} & \underline{0.565} & 0.451 & 0.563 & 0.555 & 0.556 \\
 & MATH-500 & \underline{0.682} & 0.671 & \textbf{0.685} & \textbf{0.685} & 0.671 & 0.665 & 0.668 & 0.613 & 0.628 \\
 & CNN/DailyMail & \textbf{0.648} & 0.640 & 0.633 & 0.632 & \underline{0.641} & 0.629 & 0.619 & 0.588 & 0.607 \\
 \midrule
 & \textbf{Average} & \textbf{0.637} & 0.626 & \underline{0.635} & 0.634 & 0.627 & 0.591 & 0.621 & 0.588 & 0.606 \\
\midrule
\multirow{5}{*}{0.3} & CoQA & 0.669 & 0.659 & \textbf{0.673} & \underline{0.670} & 0.648 & 0.668 & \underline{0.670} & 0.572 & 0.667 \\
 & SQuAD & \textbf{0.616} & 0.611 & \underline{0.613} & 0.611 & 0.595 & 0.547 & \underline{0.613} & 0.600 & 0.611 \\
 & MATH-500 & \underline{0.707} & 0.698 & 0.699 & 0.699 & 0.701 & 0.683 & \textbf{0.718} & 0.638 & 0.677 \\
 & CNN/DailyMail & \underline{0.664} & 0.651 & 0.651 & 0.651 & \textbf{0.669} & 0.643 & 0.642 & 0.609 & 0.621 \\
 \midrule
 & \textbf{Average} & \textbf{0.664} & 0.655 & 0.659 & 0.658 & 0.653 & 0.635 & \underline{0.661} & 0.605 & 0.644 \\
\midrule
\multirow{5}{*}{0.5} & CoQA & \textbf{0.742} & 0.704 & 0.734 & 0.728 & 0.735 & 0.733 & 0.733 & 0.557 & \underline{0.738} \\
 & SQuAD & \textbf{0.656} & 0.642 & 0.645 & 0.640 & 0.626 & 0.627 & \underline{0.647} & 0.617 & \underline{0.647} \\
 & MATH-500 & \underline{0.726} & 0.703 & 0.712 & 0.712 & 0.715 & 0.718 & \textbf{0.729} & 0.656 & 0.689 \\
 & CNN/DailyMail & \textbf{0.679} & 0.643 & 0.668 & 0.668 & \underline{0.675} & 0.668 & 0.654 & 0.615 & 0.645 \\
 \midrule
 & \textbf{Average} & \textbf{0.701} & 0.673 & 0.690 & 0.687 & 0.688 & 0.686 & \underline{0.691} & 0.611 & 0.680 \\
\midrule
\multirow{5}{*}{0.7} & CoQA & \textbf{0.749} & 0.731 & 0.738 & 0.736 & 0.735 & 0.732 & 0.728 & 0.590 & \underline{0.747} \\
 & SQuAD & \textbf{0.678} & 0.662 & \underline{0.669} & 0.666 & 0.656 & 0.658 & 0.667 & 0.579 & 0.665 \\
 & MATH-500 & \underline{0.706} & 0.700 & 0.694 & 0.691 & 0.705 & 0.701 & \textbf{0.719} & 0.664 & 0.679 \\
 & CNN/DailyMail & \underline{0.680} & 0.665 & 0.662 & 0.660 & \textbf{0.682} & 0.661 & 0.638 & 0.597 & 0.666 \\
 \midrule
 & \textbf{Average} & \textbf{0.703} & 0.690 & 0.691 & 0.688 & \underline{0.695} & 0.688 & 0.688 & 0.608 & 0.689 \\
\midrule
\multirow{5}{*}{0.9} & CoQA & 0.720 & 0.705 & \underline{0.731} & 0.730 & 0.718 & 0.718 & 0.719 & 0.592 & \textbf{0.739} \\
 & SQuAD & 0.658 & 0.650 & \underline{0.669} & 0.667 & 0.659 & 0.657 & 0.660 & 0.548 & \textbf{0.675} \\
 & MATH-500 & \underline{0.685} & 0.671 & 0.682 & 0.680 & 0.683 & 0.667 & \textbf{0.699} & 0.653 & 0.666 \\
 & CNN/DailyMail & 0.647 & 0.647 & \textbf{0.650} & \underline{0.648} & 0.647 & 0.638 & 0.641 & 0.619 & 0.633 \\
 \midrule
 & \textbf{Average} & 0.678 & 0.668 & \textbf{0.683} & \underline{0.681} & 0.677 & 0.670 & 0.680 & 0.603 & 0.678 \\
\bottomrule
\end{tabular}%
}
\end{table*}

\begin{table*}[t]
\centering
\caption{Ablation on number of generations $N$ for Qwen3-8B (AUROC; higher is better). Best and second-best per row are marked by \textbf{bold} and \underline{underline}.}
\label{tab:appendixC_qwen_n}
\scriptsize
\setlength{\tabcolsep}{3pt}
\renewcommand{\arraystretch}{0.95}
\resizebox{\textwidth}{!}{%
\begin{tabular}{llccccccccc}
\toprule
$N$ & Dataset & AvgWD M1 & AvgWD L1 & EigenWD M1 & EigenWD L1 & ER & ES & DSE & LNE & LS \\
\midrule
\multirow{5}{*}{10} & CoQA & \textbf{0.742} & 0.704 & 0.734 & 0.728 & 0.735 & 0.733 & 0.733 & 0.557 & \underline{0.738} \\
 & SQuAD & \textbf{0.656} & 0.642 & 0.645 & 0.640 & 0.626 & 0.627 & \underline{0.647} & 0.617 & \underline{0.647} \\
 & MATH-500 & \underline{0.726} & 0.703 & 0.712 & 0.712 & 0.715 & 0.718 & \textbf{0.729} & 0.656 & 0.689 \\
 & CNN/DailyMail & \textbf{0.679} & 0.643 & 0.668 & 0.668 & \underline{0.675} & 0.668 & 0.654 & 0.615 & 0.645 \\
 \midrule
 & \textbf{Average} & \textbf{0.701} & 0.673 & 0.690 & 0.687 & 0.688 & 0.686 & \underline{0.691} & 0.611 & 0.680 \\
\midrule
\multirow{5}{*}{15} & CoQA & \textbf{0.733} & 0.693 & 0.723 & 0.720 & 0.725 & 0.727 & 0.725 & 0.545 & \underline{0.728} \\
 & SQuAD & \textbf{0.670} & 0.656 & 0.653 & 0.652 & \underline{0.662} & 0.627 & 0.655 & 0.616 & 0.658 \\
 & MATH-500 & \underline{0.727} & 0.717 & 0.715 & 0.714 & 0.720 & 0.725 & \textbf{0.731} & 0.662 & 0.713 \\
 & CNN/DailyMail & \textbf{0.685} & 0.644 & 0.674 & 0.671 & 0.679 & \underline{0.683} & 0.671 & 0.643 & 0.650 \\
 \midrule
 & \textbf{Average} & \textbf{0.704} & 0.678 & 0.691 & 0.689 & \underline{0.697} & 0.691 & 0.696 & 0.617 & 0.687 \\
\midrule
\multirow{5}{*}{20} & CoQA & 0.737 & 0.698 & \textbf{0.738} & \underline{0.737} & 0.729 & 0.729 & 0.733 & 0.561 & 0.736 \\
 & SQuAD & \textbf{0.670} & 0.648 & 0.650 & 0.649 & 0.656 & 0.627 & 0.619 & 0.613 & \underline{0.667} \\
 & MATH-500 & \underline{0.728} & 0.718 & 0.717 & 0.717 & 0.724 & 0.713 & \textbf{0.737} & 0.681 & 0.719 \\
 & CNN/DailyMail & \textbf{0.685} & 0.661 & 0.680 & 0.676 & \underline{0.684} & \underline{0.684} & 0.680 & 0.658 & 0.660 \\
 \midrule
 & \textbf{Average} & \textbf{0.705} & 0.681 & 0.696 & 0.695 & \underline{0.698} & 0.688 & 0.692 & 0.628 & 0.696 \\
\bottomrule
\end{tabular}%
}
\end{table*}

\subsection{Llama-3.1-8B}
\label{app:llama_ablation}
Tables~\ref{tab:appendixC_llama_t} and~\ref{tab:appendixC_llama_n} report ablations on Llama-3.1-8B.
Compared with Qwen, EigenWD is more consistently strong on Llama models, supporting the main claim that \emph{cost-structure complexity} can be a reliable indicator when sampled responses form multi-modal inconsistency patterns.

\paragraph{Temperature $\tau$.}
Table~\ref{tab:appendixC_llama_t} shows a clear dependence on $\tau$.
Moderate temperatures often produce the best average AUROC, while extreme values degrade performance.
This matches the distribution-complexity view: if $\tau$ is too small, samples collapse and the transform-cost matrix becomes less informative; if $\tau$ is too large, costs are affected by uncontrolled randomness rather than factual inconsistency.
Across settings, EigenWD frequently remains competitive, suggesting that the spectrum of the kernelized cost matrix captures robust structural information even when sampling conditions vary.

\paragraph{Number of generations $K$.}
Table~\ref{tab:appendixC_llama_n} varies $K$ and generally shows that larger $K$ improves the reliability of transform-cost estimation.
Notably, EigenWD benefits more from additional samples than AvgWD in several cases, consistent with the intuition that estimating \emph{structural complexity} (spectrum concentration/dispersion) requires enough samples to expose multiple modes of inconsistency.

\paragraph{Practical recommendation.}
For Llama-3.1-8B, we recommend using $K\ge 15$ when feasible and a moderate temperature (e.g., $\tau\approx 0.3$--$0.7$).
This setting better reveals multi-sample inconsistency patterns and strengthens the spectral signal used by EigenWD.

\begin{table*}[t]
\centering
\caption{Ablation on sampling temperature $t$ for Llama-3.1-8B (AUROC; higher is better). Best and second-best per row are marked by \textbf{bold} and \underline{underline}.}
\label{tab:appendixC_llama_t}
\scriptsize
\setlength{\tabcolsep}{3pt}
\renewcommand{\arraystretch}{0.95}
\resizebox{\textwidth}{!}{%
\begin{tabular}{llccccccccc}
\toprule
$t$ & Dataset & AvgWD M1 & AvgWD L1 & EigenWD M1 & EigenWD L1 & ER & ES & DSE & LNE & LS \\
\midrule
\multirow{5}{*}{0.1} & CoQA & \underline{0.722} & \textbf{0.726} & \textbf{0.726} & \textbf{0.726} & 0.693 & 0.717 & 0.713 & 0.587 & 0.713 \\
 & SQuAD & 0.727 & 0.729 & \textbf{0.758} & \textbf{0.758} & \underline{0.740} & 0.725 & 0.710 & 0.618 & 0.707 \\
 & MATH-500 & 0.665 & 0.661 & \underline{0.682} & \underline{0.682} & \textbf{0.687} & 0.669 & 0.676 & 0.598 & 0.659 \\
 & CNN/DailyMail & 0.596 & \underline{0.600} & \textbf{0.601} & \textbf{0.601} & 0.598 & 0.568 & 0.597 & 0.529 & 0.573 \\
 \midrule
 & \textbf{Average} & 0.678 & 0.679 & \textbf{0.692} & \textbf{0.692} & \underline{0.680} & 0.670 & 0.674 & 0.583 & 0.663 \\
\midrule
\multirow{5}{*}{0.3} & CoQA & 0.775 & \underline{0.776} & \textbf{0.790} & \textbf{0.790} & 0.775 & 0.744 & 0.774 & 0.594 & 0.772 \\
 & SQuAD & 0.774 & 0.771 & \underline{0.799} & \textbf{0.800} & 0.770 & 0.792 & 0.762 & 0.620 & 0.733 \\
 & MATH-500 & 0.704 & 0.702 & \textbf{0.708} & \textbf{0.708} & 0.702 & 0.685 & \underline{0.705} & 0.601 & 0.685 \\
 & CNN/DailyMail & 0.615 & \underline{0.619} & \textbf{0.620} & \textbf{0.620} & 0.618 & 0.588 & 0.611 & 0.583 & 0.592 \\
 \midrule
 & \textbf{Average} & 0.717 & 0.717 & \underline{0.729} & \textbf{0.730} & 0.716 & 0.700 & 0.713 & 0.600 & 0.696 \\
\midrule
\multirow{5}{*}{0.5} & CoQA & \underline{0.771} & \underline{0.771} & \textbf{0.789} & \textbf{0.789} & 0.729 & 0.712 & 0.761 & 0.618 & 0.769 \\
 & SQuAD & \textbf{0.817} & \underline{0.815} & 0.799 & 0.799 & 0.767 & 0.785 & 0.769 & 0.613 & 0.734 \\
 & MATH-500 & 0.713 & 0.713 & \underline{0.722} & \underline{0.722} & 0.705 & 0.709 & \textbf{0.726} & 0.603 & 0.711 \\
 & CNN/DailyMail & \underline{0.625}  & 0.622 & \textbf{0.629} & \textbf{0.629} & 0.624 & 0.620 & 0.615 & 0.548 & 0.605 \\
 \midrule
 & \textbf{Average} & \underline{0.732} & 0.730 & \textbf{0.735} & \textbf{0.735} & 0.706 & 0.707 & 0.718 & 0.596 & 0.705 \\
\midrule
\multirow{5}{*}{0.7} & CoQA & 0.762 & 0.761 & \textbf{0.799} & \textbf{0.799} & 0.739 & 0.635 & 0.641 & 0.641 & \underline{0.790} \\
 & SQuAD & \underline{0.774} & 0.771 & \textbf{0.791} & \textbf{0.791} & 0.766 & 0.720 & 0.717 & 0.652 & 0.707 \\
 & MATH-500 & 0.705 & 0.708 & \underline{0.712} & \underline{0.712} & \textbf{0.717} & 0.695 & 0.711 & 0.649 & 0.690 \\
 & CNN/DailyMail & 0.617 & 0.615 & \underline{0.619} & \underline{0.619} & \textbf{0.620} & 0.593 & 0.602 & 0.557 & 0.575 \\
 \midrule
 & \textbf{Average} & \underline{0.715} & 0.714 & \textbf{0.730} & \textbf{0.730} & 0.711 & 0.661 & 0.668 & 0.625 & 0.691 \\
\midrule
\multirow{5}{*}{0.9} & CoQA & 0.723 & 0.722 & \underline{0.750} & \textbf{0.751} & \underline{0.750} & 0.634 & 0.615 & 0.633 & 0.746 \\
 & SQuAD & 0.715 & 0.713 & \textbf{0.738} & \textbf{0.738} & \underline{0.718} & 0.653 & 0.685 & 0.651 & 0.630 \\
 & MATH-500 & \underline{0.685} & 0.682 & 0.683 & 0.683 & \textbf{0.696} & 0.671 & 0.674 & 0.633 & 0.650 \\
 & CNN/DailyMail & 0.593 & 0.594 & \underline{0.602} & \underline{0.602} & \textbf{0.610} & 0.583 & 0.577 & 0.552 & 0.560 \\
 \midrule
 & \textbf{Average} & 0.679 & 0.678 & \underline{0.693} & \textbf{0.694} & \textbf{0.694} & 0.635 & 0.638 & 0.617 & 0.647 \\
\bottomrule
\end{tabular}%
}
\end{table*}

\begin{table*}[t]
\centering
\caption{Ablation on number of generations $N$ for Llama-3.1-8B (AUROC; higher is better). Best and second-best per row are marked by \textbf{bold} and \underline{underline}.}
\label{tab:appendixC_llama_n}
\scriptsize
\setlength{\tabcolsep}{3pt}
\renewcommand{\arraystretch}{0.95}
\resizebox{\textwidth}{!}{%
\begin{tabular}{llccccccccc}
\toprule
$N$ & Dataset & AvgWD M1 & AvgWD L1 & EigenWD M1 & EigenWD L1 & ER & ES & DSE & LNE & LS \\
\midrule
\multirow{5}{*}{10} & CoQA & \underline{0.771} & \underline{0.771} & \textbf{0.789} & \textbf{0.789} & 0.729 & 0.712 & 0.761 & 0.618 & 0.769 \\
 & SQuAD & \textbf{0.817} & \underline{0.815} & 0.799 & 0.799 & 0.767 & 0.785 & 0.769 & 0.613 & 0.734 \\
 & MATH-500 & 0.713 & 0.713 & \underline{0.722} & \underline{0.722} & 0.705 & 0.709 & \textbf{0.726} & 0.603 & 0.711 \\
 & CNN/DailyMail & \underline{0.625} & 0.622 & \textbf{0.629} & \textbf{0.629} & 0.624 & 0.620 & 0.615 & 0.548 & 0.605 \\
 \midrule
 & \textbf{Average} & \underline{0.732} & 0.730 & \textbf{0.735} & \textbf{0.735} & 0.706 & 0.707 & 0.718 & 0.596 & 0.705 \\
\midrule
\multirow{5}{*}{15} & CoQA & 0.775 & 0.775 & \textbf{0.801} & \textbf{0.801} & 0.735 & 0.754 & 0.771 & 0.605 & \underline{0.788} \\
 & SQuAD & \textbf{0.817} & \underline{0.812} & 0.799 & 0.799 & 0.793 & 0.785 & 0.784 & 0.689 & 0.782 \\
 & MATH-500 & 0.720 & 0.716 & 0.727 & \underline{0.728} & 0.715 & 0.705 & \textbf{0.731} & 0.692 & 0.700 \\
 & CNN/DailyMail & 0.627 & 0.627 & \textbf{0.631} & \textbf{0.631} & \underline{0.630} & 0.626 & 0.619 & 0.596 & 0.617 \\
 \midrule
 & \textbf{Average} & \underline{0.735} & 0.733 & \textbf{0.740} & \textbf{0.740} & 0.718 & 0.718 & 0.726 & 0.646 & 0.722 \\
\midrule
\multirow{5}{*}{20} & CoQA & 0.791 & 0.790 & \underline{0.813} & \textbf{0.814} & 0.751 & 0.790 & 0.780 & 0.600 & 0.792 \\
 & SQuAD & \underline{0.810} & \textbf{0.811} & 0.803 & 0.803 & 0.769 & 0.771 & 0.785 & 0.699 & 0.790 \\
 & MATH-500 & 0.721 & 0.721 & \textbf{0.728} & \textbf{0.728} & 0.708 & 0.707 & \underline{0.725} & 0.694 & 0.711 \\
 & CNN/DailyMail & 0.627 & 0.627 & \underline{0.631} & \underline{0.631} & 0.630 & 0.628 & 0.625 & 0.610 & \textbf{0.632} \\
 \midrule
 & \textbf{Average} & \underline{0.737} & \underline{0.737} & \textbf{0.744} & \textbf{0.744} & 0.715 & 0.724 & 0.729 & 0.651 & 0.731 \\
\bottomrule
\end{tabular}%
}
\end{table*}

\paragraph{EigenWD and the choice of the numerator order $p$.}
\begin{figure}[h!]
    \centering
    \includegraphics[width=0.8\linewidth]{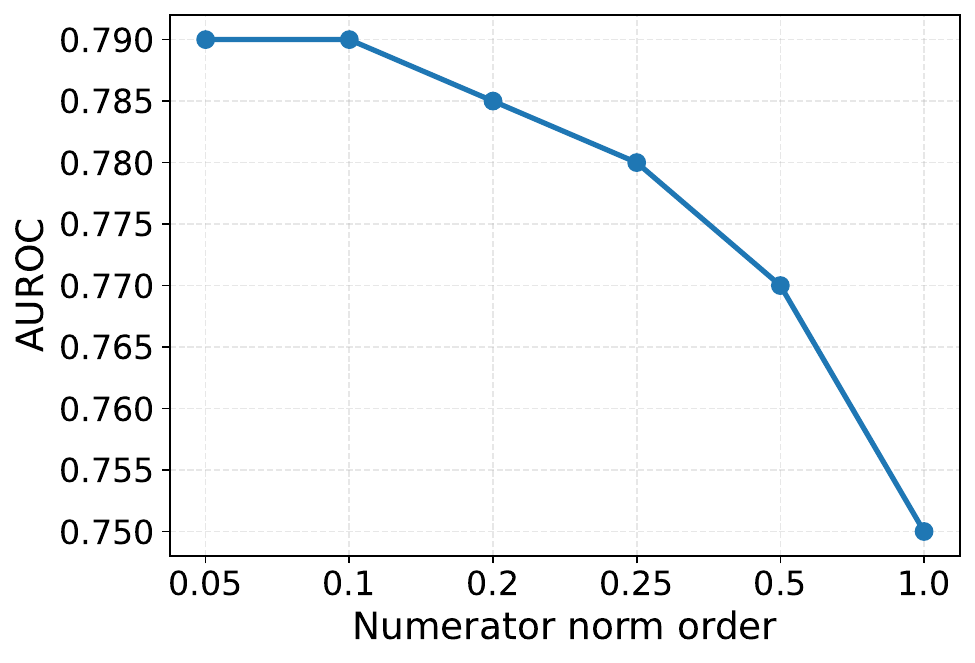}
    \caption{Ablation on the numerator order $p$ in EigenWD.}
    \label{fig:eigenwd_p_ablation}
\end{figure}
Motivated by the view in our abstract that hallucinations correlate with the \emph{complexity} of the conditional response distribution induced by a prompt, we quantify not only the average transform cost (AvgWD) but also the \emph{cost-structure complexity} across multiple sampled responses via a spectral statistic. Specifically, given the pairwise Wasserstein distance matrix $D\in\mathbb{R}^{k\times k}$ computed between token-embedding empirical measures, we kernelize it into an affinity matrix $K$ (Gaussian kernel with a median bandwidth and diagonal stabilization) and define
\begin{equation}
\mathrm{EigenWD}(x) \;=\; \frac{\|s\|_{p}}{\|s\|_{2}}, \qquad s=\sigma(K),
\label{eq:eigenwd2}
\end{equation}
where $\sigma(K)$ denotes the singular values (equivalently eigenvalues for PSD $K$) and $p\in(0,2)$ controls the sensitivity to spectral dispersion.
Let $\pi_i=s_i^2/\|s\|_2^2$ so that $\sum_i \pi_i=1$; then $\mathrm{EigenWD}(x)=\big(\sum_i \pi_i^{p/2}\big)^{1/p}$ depends only on the \emph{shape} of the spectrum (scale-invariant) and increases when the spectrum is less concentrated (i.e., many non-negligible modes remain), which corresponds to a more fragmented/multi-modal transform-cost structure across sampled responses.
Importantly, smaller $p$ makes Eq.~\eqref{eq:eigenwd2} more \emph{rank-/dispersion-sensitive}: for near-low-rank $K$ (highly consistent samples) the score stays close to its lower bound, while for spectrally spread $K$ (divergent samples) the score expands with a larger dynamic range, improving separability.
This behavior is consistent with our ablation in Fig.~\ref{fig:eigenwd_p_ablation} on CoQA, where decreasing $p$ monotonically improves AUROC (e.g., $p{=}1$ performs worst and $p\le 0.25$ is consistently better), indicating that hallucination risk is primarily reflected by \emph{how many} cost-consistency modes are activated rather than only the overall scale of costs. In all main experiments we use a small $p$ (default $p{=}0.1$) as a robust choice that captures this spectral complexity signal without additional training.

\end{document}